\documentclass[sigplan,10pt,nonacm]{acmart}
\renewcommand\footnotetextcopyrightpermission[1]{} %
\setcopyright{none}

\usepackage{booktabs}
\usepackage{graphicx}
\usepackage{xspace}
\usepackage{capt-of} %

\newcommand{\sys}{\textsc{TokTier}\xspace} %

\newcommand{\ratioColdSota}{$23.4\times$}
\newcommand{\ratioWarmSota}{$2.1\times$}
\newcommand{\ratioWarmSotaTwoM}{$3.0\times$}

\newcommand{\tok}[1]{\texttt{#1}}
\newcommand{\FE}{G}         %
\newcommand{\ME}{E}         %
\newcommand{\cat}{\,\Vert\,}
\theoremstyle{acmplain}
\newtheorem{assumption}{Assumption}
\theoremstyle{acmdefinition}
\newtheorem{remark}{Remark}

\begin{document}

\title{TokTier: Exact Stateful CPU+GPU Tokenization for Agentic LLM Serving}

\author{Zhenyu Zhang}
\affiliation{%
  \institution{Arizona State University}
  \city{Tempe}
  \state{Arizona}
  \country{USA}}
\email{zzhan641@asu.edu}

\author{Zhichao Cao}
\affiliation{%
  \institution{Arizona State University}
  \city{Tempe}
  \state{Arizona}
  \country{USA}}
\email{Zhichao.Cao@asu.edu}

\begin{abstract}
LLM serving systems cache prompt KV state, yet most front ends still
re-tokenize the full request text on every call. The mismatch is
most costly for coding agents, whose sessions repeatedly submit a
long transcript after appending a small tool result. However, reusing prior
tokenization results is hard because even a short append can change
token boundaries near the end of the previous sequence. Across
153{,}951 calls from two agent ecosystems, the median call appends only
about 1.4\,K characters, and only 1.0--3.6\% of calls start or rebuild
a session. Those calls, however, carry full contexts that reach
millions of characters. At the fleet level, the aggregate prompt-cache hit rate
is 94.1\%, and as it approaches 0.99, tokenization grows from 10\%
to 64\% of time to first token (TTFT) in our component measurements.

We present \sys, a stateful tokenization service for this two-mode
workload with CPU and GPU integration. \sys{} enforces one contract: emitted token IDs from \sys{} are always
identical to full reference tokenization of the request text. For a
session continuation, \sys{} keeps the session's previous token
sequence, re-tokenizes a small window around the append, and accepts
the splice only when a per-request check finds a stable
pre-tokenization boundary. A failed check triggers a wider window or
full reference tokenization. For a call without a reusable prefix,
\sys{} decomposes
GPT-family regex pre-tokenization into run-local rules and executes
exact pre-tokenization and BPE on a GPU for low latency and high throughput. A sampled shadow verifier
re-checks live traffic against the reference.

Across 17 production tokenizer families, our differential campaigns
include $1.50\times10^{10}$ split checks, full sweeps of a 12.4\,TB
real-text corpus, and 93{,}000+ replayed agent steps, all with zero
divergence. Incremental repair takes 0.5--1.1\,ms from 100\,K to
3\,M characters, up to \textbf{437$\times$} faster than HF tokenization and
\ratioWarmSota{} faster at 1\,M characters than the strongest
cache-based baseline (i.e., Gigatoken) in its most favorable, fully
prewarmed mode. GPU full tokenization encodes a 1\,M-character
request in 0.87\,ms, up to \textbf{491$\times$} below HF tokenization and
\ratioColdSota{} below the fastest previously published CPU method
on the same texts and protocol. With
vLLM in the loop,
\sys{} lowers median TTFT by 16--34\% in loaded
regimes and P99 TTFT by 23\% under recorded burst arrivals. Under a 50\,ms P99 objective, a
four-core CPU-based repair pool plus one GPU sustains 1{,}821 requests/s,
where a 16-core stateless CPU front end saturates at only 40 requests/s.
\end{abstract}

\maketitle

\section{Introduction}
\label{sec:intro}

A coding agent works through a sequence of LLM model calls. It reads a
file, edits it, runs a tool, observes the result, and decides what
to do next~\cite{react2023,sweagent2024}. Each call carries the
session transcript, and each tool result extends that transcript
before the next call goes out. A single user instruction can fan out
into dozens of calls. The calls arrive at machine cadence, and the
transcript grows from turn to turn.

\begin{figure*}[t]
  \centering
  \includegraphics[width=\textwidth]{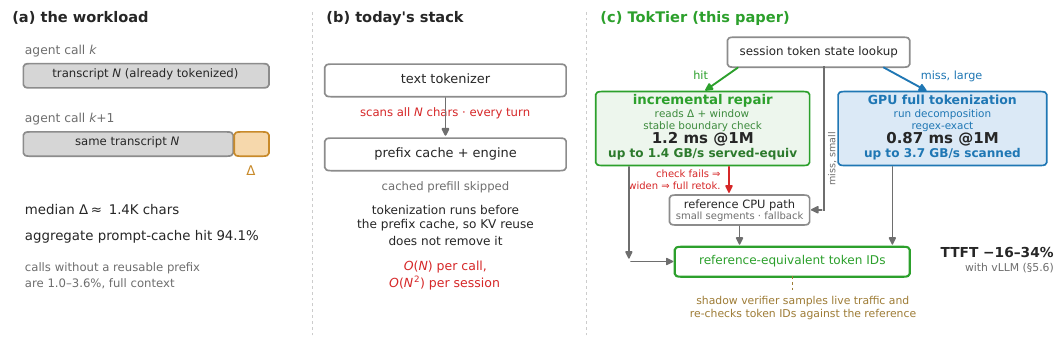}
  \caption{One agent turn through today's stack and through \sys.
  The request appends a median ${\sim}1.4$\,K characters to a
  transcript that was tokenized on previous turns, and today's front
  end re-scans the complete context before the prefix cache can act.
  \sys{} repairs session continuations around the append, routes
  rare initializations and rebuilds to an exact GPU path or the
  reference CPU path, and samples
  all outputs through a shadow verifier. Latencies are measured P50
  at 1\,M-character contexts, and the two bold figures are each
  path's best measured single-request throughput under the
  served-equivalent and scanned accountings, which are never mixed
  (\S\ref{sec:eval-warm}).}
  \label{fig:overview}
\end{figure*}

This access pattern exposes a mismatch in current LLM serving stacks.
Prefix caching (i.e., KV cache) lets the model reuse KV state for the unchanged part
of a prompt, but a stateless front end still converts the complete
request text into token IDs on every call. In the coding-agent traffic we
measure, the fleet-level prompt-cache hit rate (cached input tokens
as a share of all input tokens, pooled over the fleet) is 94.1\%. Moreover, the median
call appends about 1.4\,K characters to a context of 86\,K to
123\,K tokens. The model may process only the small uncached suffix via the KV cache,
while the tokenizer scans the entire transcript. As the prompt-cache
hit rate approaches 0.99, tokenization grows from 10\% to 64\% of
time to first token (TTFT) in our component sweeps (\S\ref{sec:workload}).

This type of workload has two distinct modes. On the one hand, most requests are session
continuations. They extend a session whose token sequence was
produced moments earlier. On the other hand, a small share, 1.0--3.6\% in our traces,
carries no reusable prefix. For example, a session starts, a compacted history is
rebuilt, or a request reaches a worker that holds no session state.
Such calls are rare, but they carry the full context and arrive in
bursts. Those two modes ask for different execution strategies to achieve high throughput and low latency during tokenization: 1) incremental repair should do work proportional to the newly appended text; and 2) full
tokenization should process a large, previously unseen context with
a low tail latency.

However, reusing token IDs across turns is harder than reusing a byte
prefix. Tokenization is not compositional. For strings $A$ and $B$,
$\mathrm{tok}(A) \cat \mathrm{tok}(B)$ can differ from
$\mathrm{tok}(A \cat B)$. The pre-tokenizer may move a piece
boundary when $B$ arrives, and BPE may then choose different merges
inside the piece. For example, a boundary inside the word ``pipeline'' turns one reference token into two different tokens
(\S\ref{sec:whyhard}). Fixed overlap heuristics do not solve the
problem, because some tokenizer rules let a change propagate past
any chosen radius. To ensure LLM serving performance and accuracy, a production system must know when a cached
prefix is safe to splice, and it must fall back when that condition
cannot be established.

Moreover, calls without a reusable prefix pose a separate problem. GPT-family
tokenizers begin
with leftmost-first regex matching, followed by BPE~\cite{gpt2}. The
regex is sequential in its usual form, because every match begins
where the previous match ends. Therefore, leveraging GPU or other hardware to achieve highly parallelized tokenization for high throughput and low latency while guaranteeing correctness is very challenging. Existing GPU tokenizers obtain
parallelism by relaxing this stage, by targeting a different
scheme, or by checking only against their own CPU
implementation~\cite{gputok2026,blockbpe2025,cudfsubword}. Token IDs are
both model input and prefix-cache keys, so a single changed ID can
alter model behavior and silently invalidate KV cache reuse.

To address those challenges, we designed and implemented \sys{} around these two
modes with one correctness contract: for every request, the emitted IDs must equal the IDs
produced by a frozen reference tokenizer (e.g., HF tokenizer) on the complete text.
\sys{} stores each live session's token IDs and byte spans. On a
session append, it re-tokenizes the new text plus a small suffix of
the old context. It compares the fresh and cached token records, and
splices only when the matched region passes a \emph{stable boundary
check}, a per-request test that the pre-tokenizer provably cannot
see across. When the check fails, \sys{} widens the window and
eventually runs the reference tokenizer on the full request. An
unsuccessful repair therefore costs latency without changing the
output, and it is rare. In replayed agent traffic, 56{,}049 of
56{,}052 appends splice on the first window, three widen it once
(0.005\%), and none falls back to full retokenization
(\S\ref{sec:eval-warm}).

For a call without a reusable prefix, \sys{} uses a GPU tokenizer
derived from an
equivalent representation of GPT-family pre-tokenization. First, the input
is classified into maximal character-class runs. Piece starts then
follow from the position within a run, a bounded amount of
neighboring text, and a few run-level summaries. This \emph{run
decomposition} removes the sequential regex scan while preserving
its output, and a GPU BPE pipeline encodes the resulting pieces.
In addition, small segments stay on the CPU, and every implementation binds to a
content-addressed tokenizer registry. Finally, a background verifier samples
live requests and compares their IDs with the reference
implementation. This guard covers bugs that depend on execution
history and therefore escape fresh-process tests.

The evaluation asks whether the design is exact, whether its two
paths match the measured workload, and whether the front-end gain
survives contact with a serving engine. We examine 17 tokenizer
families with version-pinned artifacts. Differential campaigns
include synthetic and adversarial inputs, full sweeps of a 12.4\,TB
real-text corpus, and 93{,}000+ replayed agent steps. The tested
configurations produce no divergence from the reference.
Incremental repair stays at 0.5--1.1\,ms from 100\,K to 3\,M
characters. It is up to 437$\times$ faster than HF tokenization and
\ratioWarmSota{} faster than the strongest cache-based alternative tokenizer (i.e., Gigatoken~\cite{gigatoken2026})
at 1\,M characters and \ratioWarmSotaTwoM{} at 2\,M. The GPU path
sustains 3.8--4.7\,GB/s and encodes a 1\,M-character full context
in 0.87\,ms, \ratioColdSota{} below the fastest previously
published CPU method. With vLLM in the loop, \sys{} reduces median
TTFT by 16--34\% in loaded regimes and P99 by 23\%
under recorded burst arrivals. A four-core repair pool and one GPU sustain
1{,}821 requests/s under a 50\,ms P99 objective, compared with 40
requests/s for a 16-core stateless CPU front end
(\S\ref{sec:eval}).

This paper makes three major contributions: 1) to the best of our
knowledge, it is the first to characterize session-level
tokenization in coding-agent traffic, revealing that the stream
consists of frequent small updates alongside rare, large rebuilds
(\S\ref{sec:workload}); 2) it presents an exact stateful CPU+GPU
tokenization service, up to 437$\times$ and 491$\times$ faster
than HF tokenization on its two paths, incorporating checked
boundary repair for session continuations and a run-local
reformulation of GPT-family pre-tokenization that enables full GPU
tokenization (\S\ref{sec:design}, \S\ref{sec:theory}); and 3) it
develops a
validation methodology that combines per-request checks,
version-pinned differential testing, real-text sweeps, and runtime
sampling, then evaluates the complete service in front of vLLM
(\S\ref{sec:eval}).

Our implementation is available at
\url{https://github.com/asu-idi/toktier}.

\section{Background and Workloads}
\label{sec:workload}

LLM serving studies usually summarize a request by prompt length, output
length, and arrival time. Tokenization needs one more distinction.
It must separate the text already seen in the current session from
the text that has just arrived. We write $N$ for the complete
context, $\Delta$ for the new text, and $h$ for the prompt-cache
hit ratio reported by the serving API. Full re-tokenization reads
$N$ characters. Ideal incremental repair would read close to
$\Delta$ while returning exactly the same token IDs.

\subsection{Where tokenization sits}
\label{sec:workload:sits}

An LLM serving front end receives text, applies added-token handling,
normalization, pre-tokenization, and subword encoding, then sends
token IDs to the model engine~\cite{hftokenizers,vllm2023}. Prefix caching (i.e., KV cache) begins
after those IDs
are provided to the inference engine~\cite{vllm2023,sglang2024}. A high KV-cache hit ratio therefore removes the redundant
model-side prefill work; it does not, however, remove front-end work.

The output of a tokenizer must also match the reference tokenizer used by the
model during training and previous serving, because token IDs are part of the model input and most
prefix caches index their entries by token sequence~\cite{sglang2024}. An approximate
tokenizer changes the token stream the model sees, which can significantly influence the serving accuracy and reduce KV cache
reuse. Published GPU tokenizers accept that trade (\S\ref{sec:related}).
We use the reference tokenizer's full-text output as the fundamental contract
throughout this paper and guarantee exact matches in \sys{}.

\subsection{Trace sources}
\label{sec:workload:sources}

Our primary dataset contains 153{,}951 calls from ten months of
day-to-day Claude~Code and Codex CLI use by six users on nine
machines. A local collector parses the agents' session logs, and
exports counts only (Appendix~\ref{app:workload} lists the
fields). Text is length-counted in memory and discarded, and
identifiers are HMAC-hashed with keys that never leave the source
machine. The parser also detects duplicated and replayed log
records. This check removed 26{,}578 phantom calls, 14.7\% of one
ecosystem's parsed data. Collection, consent, and parsing details
appear in Appendix~\ref{app:workload}.

We check these traces against three other independent sources.
Provider-side usage metadata for the same fleet covers
${\sim}5.12$ billion tokens of the
same kind of traffic. Here and throughout, the hit rate of a call is
the share of its input tokens the API reports as read from the
prompt cache, and a token-weighted aggregate sums both counts over
all calls before dividing. On that accounting, the fleet hit rate is
94.1\%. A public autonomous-agent trace
(\texttt{codex\_swebenchpro}) contains 20{,}230 calls from 610
successful SWE-Bench Pro
trials~\cite{swebenchpro2025,swebenchpro_traces}, whose publishers
report an aggregate hit rate of 94.2\% for that corpus. Its release preserves message
structure and message lengths, and substitutes length-preserving
filler for redacted spans (Appendix~\ref{app:workload:l3}). The TraceLab
corpus~\cite{tracelab2026} contains 357\,K steps from 4{,}265
Claude~Code and Codex sessions collected at another institution.
TraceLab omits text, so it cannot replay tokenization, but it
checks context size, append size, hit rate, and pacing externally
(Appendix~\ref{app:workload}). The sources share no collection
pipeline and no users. Where they overlap, they agree.

\subsection{Most calls add little text to a large context}
\label{sec:workload:anatomy}

\begin{figure}[t]
  \centering
  \includegraphics[width=.98\columnwidth]{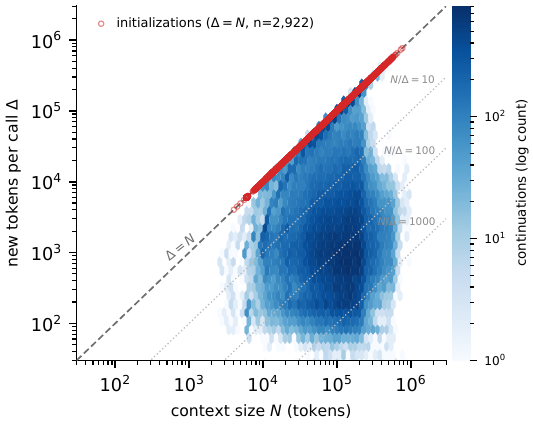}
  \caption{Joint distribution of context size $N$ and new tokens
  per call $\Delta$ over the 153{,}951 interactive calls, in the
  API's token accounting. Session continuations concentrate one to
  three orders of magnitude below the diagonal, so most calls add
  little text to a large context. Session initializations and
  rebuilds sit on the $\Delta{=}N$
  diagonal and carry complete contexts up to $10^6$ tokens.
  Marginal distributions and per-source detail are in
  Fig.~\ref{fig:workload} of Appendix~\ref{app:workload}.}
  \label{fig:workload-main}
\end{figure}

Figure~\ref{fig:workload-main} summarizes the common case. The
median append is about 1.4\,K characters in the interactive traces,
with P90 near 9--14\,K. The public autonomous trace is heavier, with
a 3.8\,K-character median, but it has the same shape. Median context
size is 86\,K (Codex) to 123\,K (Claude~Code) tokens and extends
toward a million tokens. At the same time, 74--87\% of calls have
$h{>}0.9$, and the median call has $h$ between 0.98 and 0.99.

These values expose the work amplification of full re-tokenization.
TraceLab's median step carries 126{,}180 cached-prefix tokens and
857 appended tokens. Its per-step ratio of complete context to
append has a median of 132 and a token-weighted aggregate of 23.5.
The same token-weighted ratio is about 10 on our per-call traces and
about 17 at the fleet's 94.1\% hit rate. Current text interfaces
therefore process one to two orders of magnitude more material than
the request adds.

Moreover, the agent calls also form deep loops. One user turn produces 3--5 model
calls at the median, 27--34 at P90, and 87--103 at P99. The
autonomous trace runs 30 calls per trial at the median. The tokenization cost
is paid at every step. Thus, an $O(N)$ tokenizer accumulates $O(N^2)$
work over a session whose context grows monotonically.

\subsection{Session initializations and rebuilds are rare but large}
\label{sec:workload:cold}

In trace terms, a full-context call arrives with no cached prefix
($h{=}0$) and a session continuation arrives with part of its prompt
already cached. In tier terms, a session state miss occurs when a
call reaches a worker with no reusable token state. This happens in
several cases, for example, at session start, after
history compaction, or after a migration. Full-context calls
account for only 1.0--3.6\% of calls pooled
by trace source. However, their sizes are usually very large,
differing from the typical append by two
orders of magnitude. They carry complete contexts of $10^4$ to
$10^6$ tokens and often arrive together, when many sessions start or
rebuild at once.

\subsection{Session state can outlive KV state}
\label{sec:workload:temporal}

Human pauses are long relative to default prompt-cache lifetimes.
Median gaps between user turns are 3.6--6.4 minutes in our traces.
About 55\% of Claude~Code gaps and 42\% of Codex gaps exceed the
5-minute default cache TTL. TraceLab shows the same decay from
the provider side, where the mean cached share of a step falls from 0.96 at sub-minute pauses to 0.17 beyond an hour
(Appendix~\ref{app:workload}). Token IDs and byte spans cost tens of
bytes per token, far less than KV state. A tokenization service can
therefore retain session state after the engine evicts the
corresponding KV blocks, with a very small memory overhead. When the session returns, the service
repairs the token sequence in milliseconds even if the model must
rebuild part of its KV cache. Section~\ref{sec:eval-econ} measures the
memory cost of this state and the session state hit rate that longer
retention buys (Figure~\ref{fig:memoryttl}).

\subsection{Tokenizer design requirements}
\label{sec:workload:implications}

The aforementioned measurements and analysis lead to \textbf{\emph{four design requirements}} for the new tokenizer: 1) the common path must
keep per-session token state and make its work follow the change
rather than the complete context; 2) full tokenization must absorb
large requests without creating a latency tail; 3) both paths must support
multiple frozen tokenizer versions, because 12+ model versions
appear simultaneously in our traces. And 4) every path must return
the IDs exactly the same as the reference token sequence, since the emitted IDs can influence both LLM output accuracy and KV cache effectiveness. This last requirement rules out boundary
heuristics and approximate GPU tokenizers even when their average
throughput is high.

\label{sec:workload:scissor}%
The absolute CPU cost of tokenization is manageable in some present deployments, but
its trend is unfavorable. Under the measured mixture, full
re-tokenization on a modern Rust CPU-based tokenizer costs 13.4\,ms of core
time per request, which reads as 6.7 front-end cores per 1{,}000
GPUs at 0.5 requests/s/GPU (Appendix~\ref{app:econ}). Three trends
multiply that number: 1) faster model inference raises requests
completed per GPU, 2) longer contexts increase full-tokenization time,
and 3) rising KV cache hit rates remove model-side latency while leaving front-end
work unchanged (Appendix~\ref{app:workload:gpucurve}). We treat the
trend as motivation rather than as a fleet-size prediction.
Section~\ref{sec:eval} measures the current service-time, tail, and capacity
effects directly.

\section{\sys{} System Design}
\label{sec:design}

\sys{} is a tokenization service placed between the request router
and the model engine. It accepts request text and a model
identifier, then returns exact reference-equivalent token IDs. The service
keeps token state for live sessions and selects an execution path
based on the state and the request size.
Figure~\ref{fig:arch} shows the data and workflow.

\begin{figure}[t]
  \centering
  \includegraphics[width=.96\columnwidth]{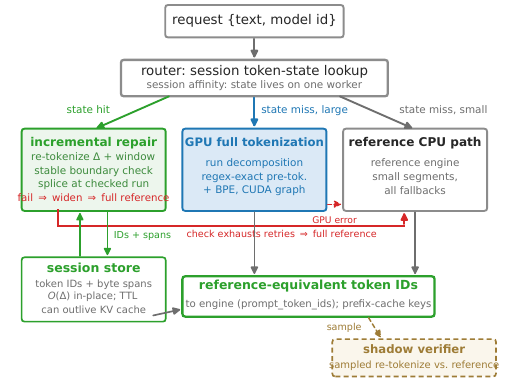}
  \caption{Request lifecycle through \sys. The router checks for
  live session token state, then sends session continuations to
  boundary repair, large state-miss segments to the GPU path, and
  small segments or any fast-path failure to the reference CPU
  path. The session
  store keeps token IDs and byte spans per live session. A shadow
  verifier re-tokenizes a sampled fraction of emitted IDs against
  the reference engine and quarantines divergences.}
  \label{fig:arch}
\end{figure}

Each tokenizer version is registered by content hash. Its entry
carries the added-token rules, normalization configuration,
pre-tokenization rules, vocabulary, merge table, and the
family-specific checks used by the repair and GPU paths. A request
is never interpreted through ambient process state, and a session
created under one tokenizer version is never repaired under
another.

The router classifies each request by session state. A
\emph{session state hit} finds a compatible session record, one
whose text prefix matches the new request. The record holds token
IDs and source byte spans from the previous call, and \sys{}
repairs this record near the change. If a \emph{session state miss}
finds no compatible record, the request is a session
initialization or a history rebuild. Therefore, its context is tokenized in
full. The router decides the request state inside the tokenization tier,
independently of the engine-side prompt-cache reuse $h$ of
\S\ref{sec:workload}. The two layers hold separate state with
separate lifetimes. In \sys{}, large state-miss segments go to
the GPU path of \S\ref{sec:theory}. Small segments use the
reference CPU tokenizer, where kernel-launch overhead would
dominate if we still used the GPU. Every path is held to one
output contract: \textbf{the emitted token IDs are always identical
to full reference tokenization of the request text.} Every
mechanism below may fail toward more work, never toward different
IDs.

\subsection{Why an append can change old tokens}
\label{sec:whyhard}

Modern BPE tokenizers apply two stages. A pre-tokenizer divides
text into short \emph{pieces}, usually with a regex, and BPE then
encodes each piece independently. Thus, an arbitrary text boundary is
not a token boundary.

\begin{figure}[t]
  \centering
  \includegraphics[width=.86\columnwidth]{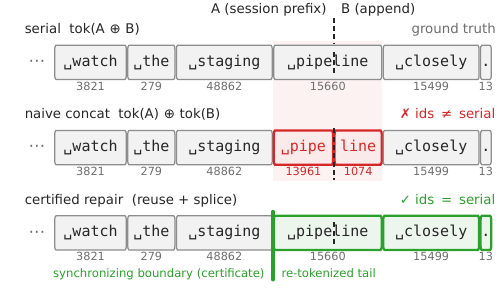}
  \caption{A boundary inside the word ``pipeline'' changes the token
  sequence (Llama-3.1-8B tokenizer, real token IDs below each box).
  Independent tokenization produces two tokens
  (\texttt{\textvisiblespace pipe}+\texttt{line}) where full
  tokenization produces one (\texttt{\textvisiblespace pipeline}).
  Repair re-tokenizes the affected region and reuses the cached
  prefix only after finding a stable boundary, reproducing the
  serial stream bit for bit.}
  \label{fig:seam}
\end{figure}

Figure~\ref{fig:seam} gives a real example. The previous request
ends after ``pipe'', and the next request appends ``line''. The
reference tokenizer sees ``\textvisiblespace pipeline'' as one piece
and emits one vocabulary token. Tokenizing the two sides
independently emits ``\textvisiblespace pipe'' and ``line'' as two
tokens. The appended text changes the trailing token of the old
prefix and shifts every later position. The drift will influence
the LLM output accuracy and silently invalidate KV cache reuse for
the rest of the session (since the token-ID sequence is the
prefix-cache key).

Two mechanisms create this problem. First, when text arrives on the
right, the pre-tokenizer can move a piece boundary. Second, the BPE
merge order inside the resulting piece can then change. A fixed overlap radius
around the append is not a correctness rule. Digit grouping,
whitespace lookahead, and newline absorption can push the effect
past any chosen radius, and our adversarial oracle produced
cascading counterexamples for every bounded-radius rule we
formulated. We
therefore use a window only to \emph{search} for a reusable
boundary, and base acceptance on a separate per-request check
derived from the tokenizer family. The published track record shows the
difficulty is routinely underestimated. LoPT, the one peer-reviewed
segmented tokenizer, specifies an experimental configuration that
does not satisfy its
safety theorem's stated precondition~\cite{lopt2026}. Gigatoken, a
high-throughput open-source engine, documents boundary stability as
an engineering assumption. Our full-corpus differential evaluation
over 11 tokenizer families finds mismatches against the reference
implementation at a rate of $1.3\times10^{-8}$ (558 divergent
document--tokenizer pairs in $4.18\times10^{10}$ comparisons); every
archived specimen traces to a difference in the Unicode data
versions the two engines carry rather than to segmentation, which is
precisely the kind of version-relative discrepancy an
assumption-based design cannot certify away~\cite{gigatoken2026}.

\subsection{Incremental repair}
\label{sec:warm}

\begin{figure}[t]
  \centering
  \includegraphics[width=.86\columnwidth]{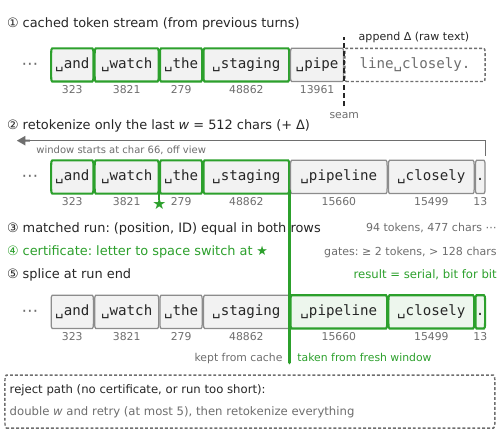}
  \caption{Incremental repair on the running example of
  Figure~\ref{fig:seam} (same frozen tokenizer, every token ID is
  real engine output). The service re-tokenizes the append and a
  suffix of the stored context, matches fresh and cached token
  records, checks that the equal run contains a stable
  pre-tokenization boundary (the starred class transition, called a
  certificate in the formalization of Appendix~\ref{app:proof}),
  and splices at the end of that run. A failed check widens the
  window and eventually falls back to full reference tokenization.}
  \label{fig:protocol}
\end{figure}

Incremental repair updates the previous token sequence instead of
rebuilding it. Suppose a session stored the tokenization of text
$A$ and the next request is $A \cat B$. The service takes the last
$w$ characters of $A$, with $w = 512$ by default, appends $B$, and
tokenizes this window with the same reference-compatible engine
used for the session. It then compares the fresh records with the
cached records that overlap the old part of the window.

Figure~\ref{fig:protocol} shows the full operation. \sys{}
finds the longest run of records on which the cached and fresh
sequences agree in both position and token ID. The \emph{stable
boundary check} then admits the run through three conditions: 1) the
run must contain at least two tokens; 2) it must cover more normalized
characters than the longest token the vocabulary can emit (probed
per tokenizer; 128 for the Llama family), which rules out an
accidental single-token match; and 3) most importantly, the run must
contain a family-specific stable boundary, a character-class
transition at which the pre-tokenizer's output to the right
provably does not depend on text to the left. A matched ID run
alone is not sufficient, because context-dependent digit grouping
defeats it, a failure our certificate-level oracle exposed
adversarially. Appendix~\ref{app:proof} formalizes the stable
boundary as a \emph{synchronizing boundary} and proves the splice
theorem.

When the conditions hold, cached records are kept before the
matched run and fresh records after it. The implementation
splices at the end of the equal run, which
Corollary~\ref{c:shift} shows is interchangeable with splicing at
the stable boundary inside it. The resulting ID sequence equals
full tokenization of $A \cat B$. When a condition fails, \sys{} doubles $w$ and retries, at most five times, then sends the
complete request to the reference engine. The search can miss an
available boundary, but it cannot accept an invalid token sequence output. A miss
therefore costs extra work and never changes the returned IDs. The
same procedure applied on both sides of an edit handles mid-context
mutations, and large rewrites that fail both searches are tokenized
in full. Misses are rare in practice: in replayed agent traffic,
56{,}049 of 56{,}052 appends splice on the first window, three widen
the window once, and none reaches the full-retokenization fallback
(\S\ref{sec:eval-warm}).

\subsection{Why the splice check is sufficient}
\label{sec:splicesound}

The formal argument is short at the system level. The pre-tokenizer
maps the input into an ordered sequence of pieces, and BPE encodes
each piece without state from other pieces. At a stable boundary,
the piece sequence to the right is independent of the left context.
If the cached and fresh records agree across a run that contains
such a boundary, both executions have reached the same piece
sequence and the same BPE output. Cached records can therefore
serve the left side, and freshly tokenized records serve the right.

The proof is parameterized by a frozen tokenizer configuration. We
establish which character-class transitions synchronize each
supported family and check that the configuration satisfies the
proof assumptions. 15 of the 17 families we examined meet
these conditions. The other two are provably outside the predicate
class (one normalizer erases whitespace structure entirely, leaving
no internal boundary to repair against) and always take the
full-retokenization path.

Although the theorem covers the family-level boundary abstraction, it does
not prove the evolving Rust, Python, and CUDA implementation. We
validate the implementation with version-pinned differential
campaigns against the reference tokenizer, and we sample deployed
outputs at runtime. Section~\ref{sec:eval-correctness} separates
these forms of evidence and reports their coverage, and
Appendix~\ref{app:guarantee} places the two nearest neighbors on
the resulting guarantee ladder. In general, across all differential
campaigns and deployed sampling to date, this validation has
observed zero divergence on the shipped configurations
(\S\ref{sec:eval-correctness}).

\subsection{State management}
\label{sec:state}

The session store keeps token IDs, byte spans, the tokenizer hash,
and a compact index from text positions to token records in memory. Repair
edits these arrays in place, and derived indices are updated
lazily, so the bookkeeping cost follows the repair window rather
than the full context. An earlier
implementation rebuilt offset arrays on every turn and silently
recovered $O(N)$ behavior even though tokenization touched only the
window. With it fixed, the protocol does $O(\Delta + w)$ work
end to end, independent of $N$.

\sys{} reclaims the state at session end. Therefore, memory usage depends only on live sessions
rather than content history, and it can achieve a longer lifetime than
KV cache (\S\ref{sec:workload:temporal}). A tokenizer worker owns the
sessions assigned to it. Requests that lose affinity become session
state misses and remain correct. Cross-worker state transfer is an
optimization that the current system does not require (planned in our future work).

\section{Exact GPU Tokenization}
\label{sec:theory}

Session initializations and history rebuilds contain no reusable
session state. Their contexts are
large enough to trigger long CPU service times, leading to high latency, especially when
several sessions start or rebuild together. To address this
challenge, \sys{} moves full tokenization to a GPU for high-performance parallel processing, which also guarantees exact equivalence to the reference token sequence.

\subsection{Removing the sequential regex scan}
\label{sec:rundecomp}

GPT-family tokenizers specify pre-tokenization as a leftmost-first
alternation regex with backtracking~\cite{gpt2}. In the direct
implementation, each match begins at the previous match's end. This
dependency prevents independent matching of arbitrary chunks for parallel processing, and
splitting the text first would recreate the seam problem of
\S\ref{sec:whyhard}. Prior GPU tokenizers avoid the tension by
weakening this stage or by targeting a different
specification~\cite{gputok2026,blockbpe2025,cudfsubword}. Therefore, none of
them reports reference-exact token-ID agreement, which our output
contract requires.

\begin{figure*}[t]
  \centering
  \includegraphics[width=\textwidth]{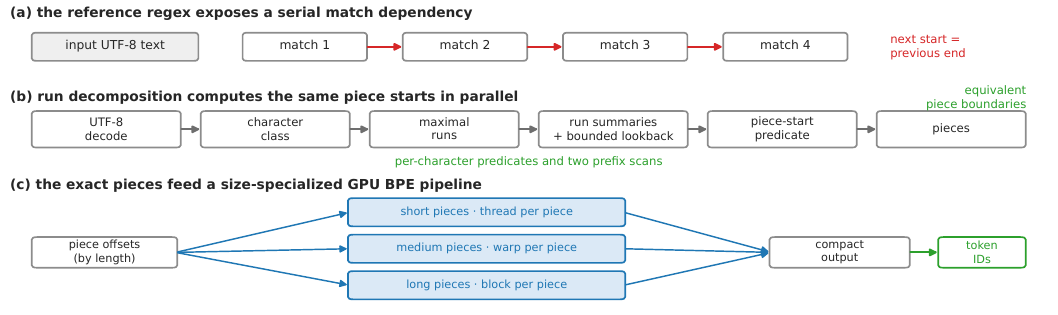}
  \caption{Run decomposition. The reference regex exposes a serial
  dependency between matches. The equivalent formulation computes
  character classes, maximal runs, and a local piece-start predicate
  with parallel passes. The resulting pieces feed size-specialized
  BPE kernels.}
  \label{fig:rundecomp}
\end{figure*}

Our central observation is that the production patterns we study
never needed a backtracking engine. We classify each character into
one of four classes (letter, number, whitespace, other) and form
maximal same-class runs. Whether a character begins a piece then
depends only on its offset within the run, at most four characters
of lookback across the run boundary, and three per-run aggregates
(run start, first non-CRLF position, last CRLF position). Character
classification can be applied in parallel, run summaries come from prefix scans,
and the piece-start test is an independent per-character predicate
(Figure~\ref{fig:rundecomp}).

We derive the rules one regex alternative at a time. The cl100k
lineage maps directly onto the four classes and local run rules.
The o200k lineage adds a case-boundary rule and a sparse
contraction path, and Unicode marks take a narrow sparse fallback
while the common path stays vectorized. DeepSeek-family
tokenizers~\cite{deepseekv3,deepseekv3tok} apply three splitters in
sequence.
Their composition folds into the same per-character predicate,
provided each splitter re-splits every piece of the previous stage
and unmatched spans survive as implicit pieces.

The derivation makes the regex-to-run-rule mapping concrete in two
steps. First, for each alternative of the reference regex, it names
the run rule that reproduces that alternative's matches. Second, it
pins the tables under which this correspondence holds. One is the
tokenizer configuration, frozen by content hash. The other is the
set of Unicode character-class tables, probed directly out of the
reference engine rather than taken from library data
(\S\ref{sec:eval-correctness}). A change to either table therefore
invalidates the derivation instead of silently changing its
meaning.

Differential testing then checks the implementation at two levels,
the piece boundaries produced by pre-tokenization and the final
token IDs (\S\ref{sec:eval-correctness}). Both levels are needed,
because the ID level can hide a boundary failure. Suppose a faulty
rule splits one piece into two halves and no merge in the
vocabulary spans the split point. BPE then encodes the two halves
into the same ID sequence that the intact piece would produce, and
an ID-level comparison passes on a wrong piece division. The same
fault becomes visible only when an input happens to place a
mergeable pair across that seam, so an ID-only harness finds it by
accident rather than by construction. Boundary-level comparison
removes that dependence on the input.

\subsection{GPU BPE}
\label{sec:kernels}

After pre-tokenization, each piece is encoded independently. Merge
ranks, vocabulary entries, and byte strings are packed into GPU
hash tables. Two equivalences license a parallel schedule. A
round's minimum merge rank identifies one pair value, and merging
all leftmost non-overlapping occurrences of that pair equals
iterated single merges. Only pairs adjacent to a completed merge
must be probed again.

Work is dispatched by piece length, so that each piece is encoded
by the smallest execution unit able to hold its state. Pieces up to
32 bytes, the vast majority in natural text, run thread-per-piece
with the sequence in registers. Pieces of 33--128 bytes run
warp-per-piece: the 32 lanes probe candidate pairs in parallel, a
shuffle reduction finds the minimum rank, a ballot materializes the
hit positions as a bitmask, and the leftmost non-overlapping
selection is then evaluated in closed form on that mask. Longer
pieces fall to a block-per-piece kernel with a shared-memory
candidate bitmap. Families configured with \texttt{ignore\_merges}
(Llama~3, gpt-oss) admit one further shortcut: their reference
encoder emits a single ID whenever the whole piece is already a
vocabulary entry, so the kernel probes the vocabulary first and
skips the merge loop on a hit.

The complete pre-tokenization path consists of on-device UTF-8
decoding, character classification, run construction, the rule
predicate, and two prefix scans. Multi-document batches
concatenate inputs and cut runs at document boundaries, at
0.36\,\textmu s per small document. Pre-tokenization alone reaches
30--77\,GB/s on one RTX~PRO~6000, depending on script mix. The
complete path, including BPE and output compaction, sustains a
multi-document batch throughput of 3.8--4.7\,GB/s
(876--1206\,Mtok/s) across English, CJK, and
templated corpora, byte-identical to the reference.

\subsection{Single-request path}
\label{sec:singlereq}

Batch throughput does not determine the latency of a single
state-miss request. A
conventional GPU pipeline reads intermediate counts back to the
host before choosing the next launch. Those synchronizations are
invisible in batch benchmarks and only visible in a single request,
especially when the host is busy. \sys{} keeps piece counts,
dispatch lists, and output length in device memory. Kernel geometry
depends on buffer capacity rather than host-visible counts, and the
whole bytes-to-IDs chain is captured as a CUDA graph over a small
set of input-size buckets. Host synchronizations drop from eight
per request to two. Under contention from 28 competing CPU
processes, the graph path keeps P99 near its idle value, moving
from 0.38 to 0.49\,ms on a 50\,K-character request, while the eager
pipeline's P99 doubles from 0.52 to 1.05\,ms
(\S\ref{sec:eval-cold}).

\subsection{Scope and fallbacks}
\label{sec:coldpath}

The GPU path covers the tokenizer body for the cl100k, o200k, and
DeepSeek pattern families evaluated in this paper. Added-token
literal extraction runs before family dispatch. NFC normalization
stays on the CPU for the one family that requires it. The GPU path
currently returns IDs without source byte spans, so a state-miss
request that must initialize session state for later repair uses
the reference tokenizer to obtain spans. Unsupported tokenizers and requests that fail a
family check remain on the CPU reference path. In addition, segments above a
2\,KB threshold go to the GPU while smaller ones stay on the CPU.
The router also spills to the CPU under GPU backpressure. These
cases may change latency and capacity, but the output contract is
unchanged.

\subsection{Service implementation}
\label{sec:integration}

The session store and repair logic are implemented in Rust.
Tokenizer entries are immutable after registration, and workers use
session affinity so a session continuation reaches the state its
previous turn created. Added and special tokens are literal strings that
the model reserves as single IDs, \texttt{<|endoftext|>} for
example. A leftmost-longest literal pass extracts them before the
tokenizer family handles the remaining segments.

\paragraph{Reference CPU path.}
\label{sec:cpufloor}
This path is the frozen HuggingFace reference implementation. It
serves the small segments, unsupported families, repair fallback,
backpressure spill, and state rebuilds that need source spans. It
carries under 0.3\% of request characters across our serving mixes,
so reference speed suffices. The tier also ships the fastest
available Rust tokenizer, but strictly as a timing baseline. Our
verifier caught exactly this class of engine returning
history-dependent IDs (\S\ref{sec:eval-verify}), so speed and
correctness anchoring are assigned to different engines by
construction. The repair-window engine itself is pluggable. We run
the same protocol unchanged over the Python reference stack and a
Rust crate, with the certification replay campaigns of
\S\ref{sec:eval-correctness} as the admission bar.

\paragraph{Engine interface.}
\sys{} submits token IDs through vLLM's
\texttt{prompt\_token\_ids} interface. We verified that
prefixes inserted through text and through IDs share the same vLLM
prefix-cache key space. Text-written prefixes are hit by
ID-submitted requests at a 99.84\% rate (99.94\% in reverse) with
identical outputs, so the engine needs no modification.
Instrumenting vLLM~0.25 also located the original tokenization work
in the API-server process before engine scheduling, which is why
removing it changes TTFT.

\paragraph{Shadow verification.}
\label{sec:shadow}
A background thread samples emitted sequences and re-tokenizes the
same text with the reference engine. The verifier compares IDs
exactly and quarantines any mismatch with a content hash for
offline reproduction. If it falls behind, it drops samples and
increments a visible counter. The sampling rate is 5\% in the serving experiments and
100\% in offline sweeps. Runtime sampling covers history-dependent
implementation failures that no fresh-process test can exercise
(\S\ref{sec:eval-verify}).

\section{Evaluation}
\label{sec:eval}

We first present the overall comparison against every measured
baseline, then test both execution paths against frozen reference
tokenizers. We measure
incremental repair and GPU full tokenization separately, followed by evaluating
burst tails, capacity, and TTFT with vLLM. The final
experiments cover resource use and runtime verification. What
limits the GPU path is discussed with the other limitations in
\S\ref{sec:discussion}.

Experiments run on a dual-socket AMD EPYC~9115 host (32 physical
cores, SMT disabled) with four RTX~PRO~6000 Blackwell GPUs of
96\,GB each. CPU experiments are NUMA-pinned. One GPU serves
vLLM~0.25 with prefix caching enabled and one serves \sys{} unless
stated otherwise. Tokenizer artifacts, datasets, and dependency
versions are frozen by content-addressed manifests. CPU one-shot
latency uses one measured encode per process, because Rust
tokenizers retain word caches across calls.

The reference is the HuggingFace fast tokenizer for each frozen
model. The main CPU performance baseline is
\emph{fastokens}~\cite{fastokens}, a separately distributed Rust
tokenizer that vLLM supports as an opt-in backend. We also compare \sys{} with
Gigatoken~\cite{gigatoken2026}, a cache-based high-performance CPU
tokenizer, and with our clean-room LoPT
reproduction~\cite{lopt2026}. Approximate GPU tokenizers appear only in correctness
comparisons, because they do not satisfy the output contract.

\subsection{Performance overview}
\label{sec:eval-overview}

\begin{figure*}[t]
  \centering
  \includegraphics[width=.92\textwidth]{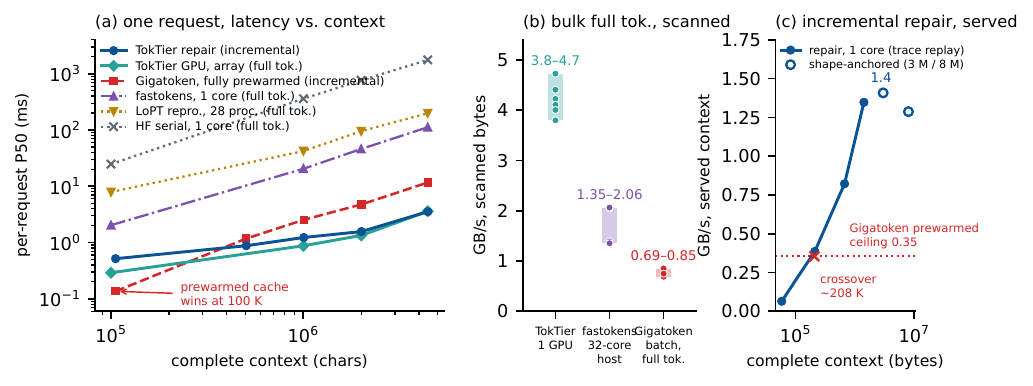}
  \caption{Performance against every measured baseline.
  (a)~Single-request P50 latency versus complete context on
  identical real texts. Incremental-repair lines price one append on
  a live session under the protocol of \S\ref{sec:eval-warm}, and
  full-tokenization lines price a full encode of a fresh request
  under the protocol of \S\ref{sec:eval-cold}. Gigatoken runs in its most favorable,
  fully prewarmed mode and wins below the 100\,K--500\,K
  crossover. (b)~Bulk full tokenization in the scanned-bytes
  account, with per-corpus measured points. The Gigatoken column
  is its batch mode on the same host with an empty cache.
  (c)~The served-context account for one incremental-repair core
  versus complete context, on the public-trace replay of
  Appendix~\ref{app:equiv}, with shape-anchored points at the 3\,M
  and 8\,M shapes. The dotted line is Gigatoken's prewarmed ceiling
  in the same served account, crossed near 208\,K bytes. The
  scanned and served accounts sit on separate axes and are never
  mixed (\S\ref{sec:eval-warm}).}
  \label{fig:headline}
\end{figure*}

Figure~\ref{fig:headline} places \sys{} against every measured
baseline before the detailed protocols. On session continuations, \sys{} repair
prices one append at 0.5--3.5\,ms P50 from 100\,K to 4.4\,M
characters and stays near flat, while every full retokenizer grows
with the context. The strongest cache-based baseline, Gigatoken, wins below
its 100\,K--500\,K crossover and loses by a growing margin beyond
it. On session initializations, the GPU path answers a fresh request in
0.29--3.59\,ms over the same span, at least 7$\times$ below every
CPU full-tokenization baseline at every shape. In bulk full
tokenization, one GPU sustains 3.8--4.7\,GB/s, against 1.35--2.06\,GB/s for the
strongest 32-core CPU configuration and 0.69--0.85\,GB/s for
Gigatoken batch on an empty cache, while one \sys{} repair core serves up
to 1.4\,GB/s of context under the separate served account. The
rest of this section establishes exactness first, then details the
protocols and baselines behind each line (\S\ref{sec:eval-warm},
\S\ref{sec:eval-cold}).

\subsection{Exactness}
\label{sec:eval-correctness}

\begin{table}[t]
  \centering\small
  \caption{Main correctness campaigns. Split rows compare every
  pre-tokenization boundary, and end-to-end rows compare final
  token IDs. Every comparison is against the reference
  implementation.}
  \label{tab:correctness}
  \begin{tabular}{@{}p{2.2cm}p{3.1cm}r@{\hspace{5pt}}r@{}}
    \toprule
    Path & Inputs & Checks & Div. \\
    \midrule
    GPU split-level & synthetic, adversarial, and a full 12.4\,TB
    real-text sweep; four pattern families & $1.50{\times}10^{10}$ & 0 \\
    GPU end-to-end & six production tokenizers; synthetic,
    adversarial, and real documents & $6.21{\times}10^{7}$ & 0 \\
    Incremental repair & two public agent corpora and 15{,}000 adversarial
    edits & $1.09{\times}10^{5}$ & 0 \\
    Shadow sample & offline, capacity, and vLLM runs &
    $>5{\times}10^{4}$ & 0$^{*}$ \\
    \bottomrule
  \end{tabular}
  \\[2pt]{\footnotesize $^{*}$the verifier also exposed one genuine
  bug in an external tokenizer (\S\ref{sec:eval-verify}).}
\end{table}

Table~\ref{tab:correctness} summarizes the differential campaigns.
Every admitted configuration reports zero divergence. The
split-level campaigns are multiplicative rather than sharded. Each
of the four GPU-path pattern families ran its own sweep of the
12.4\,TB corpus (Nemotron-CC~v2.1, distributed as 4.59\,TB of
compressed archives, and corpus sizes throughout this paper are
decompressed text), every document for two families and
$3.7\times10^{9}$ documents each for the other two.

Split-level testing found an implementation bug that final IDs did
not expose reliably. An o200k fallback path missed a run-head
condition and could fuse two pieces across a chunk seam. A
six-character input reproduces the failure. Smaller synthetic suites had
passed clean, and the bug appeared within one minute of the
million-input mixed campaign. We fixed it and added regression
pins.

Real text found a different class of problem. Character-class
tables that three components had derived from stdlib data
(Unicode~15.0) disagreed with the reference engine's own tables
(16.0) on ${\sim}9.7$\,K codepoints. Synthetic generators built from
the older tables could never emit those characters, and two hours
of CommonCrawl-derived text exposed the skew that $6\times10^{7}$
synthetic checks could not. Every class table is now probed
directly out of the reference engine. The corpus-scale end-to-end
comparator matched the unmodified reference encoder bit for bit on
$3.7\times10^{6}$ documents before it judged anything, and it
caught all planted faults used to validate the harness. The
DeepSeek group (V3,
V4-flash, HY3) then passed the same ladder with zero divergence.
That ladder comprises $10^{7}$ synthetic split checks, end-to-end
suites on all three variants and dispatch paths
($1.28\times10^{5}$ checks), and one full-corpus split sweep of
every document ($3.80\times10^{9}$ checks). The three variants
share the splitter configuration hash, so one sweep judges the
group.

Incremental repair replays 74{,}064 repairs from public SWE-smith
trajectories and 19{,}620 from the public autonomous-agent trace,
with zero ID divergence. Re-replaying both corpora under the
shipped boundary-check configuration accepts every splice the
length-only check accepts, 92{,}484 certified splices with zero
violations and indistinguishable latency. It also applies 15{,}000
adversarial mid-context edits targeting digit runs,
ideograph--punctuation seams, and fraction characters, with
identical behavior. Unsupported families and failed boundary checks
take the fallback path and are excluded from the accepted count.

\subsection{Incremental repair: latency and throughput}
\label{sec:eval-warm}

\begin{figure}[t]
  \centering
  \includegraphics[width=.82\columnwidth]{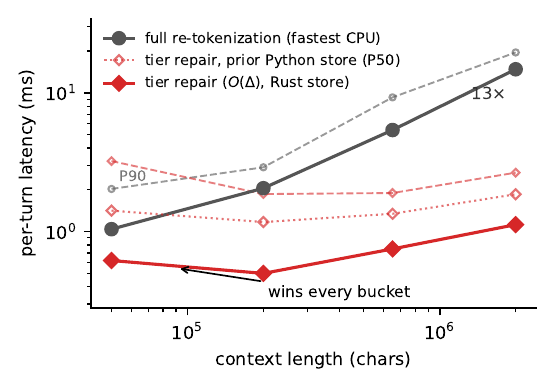}
  \caption{Incremental repair latency versus complete context length on
  the public-trace replay (log--log, solid P50, dashed P90). The
  Rust session store keeps \sys{} repair latency at 0.5--1.1\,ms from 100\,K to
  3\,M characters, while full CPU re-tokenization grows with
  context size. The dotted line is the prior Python-store
  implementation on the same protocol.}
  \label{fig:heal}
\end{figure}

Figure~\ref{fig:heal} measures 12{,}674 real append steps under the
o200k family. The shipped Rust store, which replays the
certification battery of \S\ref{sec:eval-correctness} with zero
divergence, keeps median \sys{} repair latency between 0.5 and 1.1\,ms
from 100\,K to 3\,M characters. Full re-tokenization on fastokens
with a prewarmed word cache grows with the complete context and loses
in every bucket. \sys{} repair is 1.7$\times$ faster below 100\,K
characters, 4.1$\times$ from 100\,K to 300\,K, 7.2$\times$ from
300\,K to 1\,M, and 13$\times$ from 1\,M to 3\,M. The comparison
leans the baseline's way, since fastokens re-timed with an empty
word cache runs 2.6--3.0$\times$ slower than the archived curve we
plot. The
previous Python session store measured 1.2--1.9\,ms with mild
growth, because it copied arrays during each splice. In-place Rust
bookkeeping recovered the intended $O(\Delta+w)$ scaling.

\begin{figure*}[t]
  \centering
  \includegraphics[width=.86\textwidth]{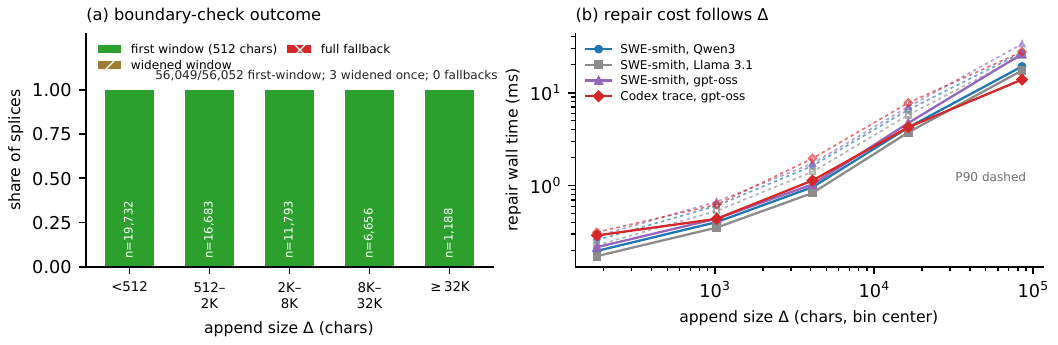}
  \caption{Boundary-check behavior over 56{,}052 replayed real
  splices (SWE-smith streams under three families, plus the public
  Codex trace). (a)~The default 512-character window accepts
  56{,}049 splices on the first attempt at every append size, three
  splices widen the window once, and none falls back to full
  retokenization. (b)~\sys{} repair wall time grows with the append and
  not with the context, as $O(\Delta+w)$ predicts. Latency here is
  replay-measured in one process and supports shape comparisons
  only.}
  \label{fig:boundary}
\end{figure*}

Fast-path coverage is a separate question from speed, so we
instrument the same replays for per-splice window outcomes
(Figure~\ref{fig:boundary}). The first 512-character window accepts
99.995\% of 56{,}052 real splices, across append sizes from tens of
characters to 213\,K. Three splices widen the window once, and none
reaches the full-retokenization fallback. Adversarial inputs can
force the fallback by construction. The observed agent traffic does
not.

\begin{table}[t]
  \centering\small
  \caption{Session-continuation append latency (ms) on identical Qwen3 session
  texts, same host. One timed call per sample, $n{=}24$ per shape,
  median append 1.5--1.6\,K characters. Context sizes (ctx) are in
  characters. Gigatoken runs in its most favorable mode, with its
  per-object cache fully prewarmed on the session prefix. The 3\,M
  and 8\,M shapes of the same protocol appear in the text and in
  Appendix~\ref{app:delta}.}
  \label{tab:warmhead}
  \begin{tabular}{@{}lrrrrr@{}}
    \toprule
    & \multicolumn{2}{c}{repair (ours)} &
      \multicolumn{2}{c}{Gigatoken prewarmed} & HF serial \\
    ctx & P50 & P90 & P50 & P90 & P50 \\
    \midrule
    100\,K & 0.52 & 0.87 & 0.14 & 0.22 & 23.6 \\
    500\,K & 0.88 & 1.68 & 1.18 & 1.25 & 165.2 \\
    1\,M   & 1.23 & 2.83 & 2.53 & 2.64 & 318.0 \\
    2\,M   & 1.57 & 4.60 & 4.76 & 5.25 & 658.8 \\
    4.4\,M & 3.54 & 10.01 & 11.66 & 12.11 & 1548.0 \\
    \bottomrule
  \end{tabular}
\end{table}

Table~\ref{tab:warmhead} prices one continuation append on identical real
session texts. Gigatoken wins at 100\,K characters, and the
crossover lies between 100\,K and 500\,K. At 1\,M characters repair
is \ratioWarmSota{} faster at the median, at 2\,M it is
\ratioWarmSotaTwoM{} faster, and the lead widens with context
(3.4$\times$ at 3\,M, 3.3$\times$ at 4.4\,M). Gigatoken's latency
keeps growing because even a full cache hit rescans the complete
context, while \sys{} repair scans the append and the window. Its cache is
also process-lifetime memory keyed on content, where \sys{} repair state
is per-session and freed. Fastokens takes 9.7\,ms at 1\,M
characters on the same inputs and stays 7.9--16$\times$ slower than
\sys{} repair from 1\,M upward. On
an 8\,M-character headroom shape beyond any deployed window (median
2.04\,M tokens), \sys{} repair still leads at P50 (6.25 vs.\ 23.42\,ms),
and the first repair after each session bootstrap forms the entire
P90 tail at 26--32\,ms. Appendix~\ref{app:delta} sweeps the append
size from 1\,K to 100\,K characters at every shape. \sys{} repair cost
grows with the appended bytes at 2.9--3.9\,MB/s, as $O(\Delta)$
predicts, and the fully prewarmed cache catches up only past the
measured append distribution's 99th percentile of 38\,K characters.
Appendix~\ref{app:routes} evaluates three exploration routes that
reclaim that large-append regime.

The context sizes above are characters, not tokens. At the measured
3.9--4.2 characters per token on these texts, the 4.4\,M-character
row corresponds to ${\approx}1.14$\,M tokens (minimum 1.08\,M) and
covers the million-token context windows now in deployment. A 1{,}050{,}000-token
window is in deployment~\cite{gpt56sol}, and $10^{6}$-token windows
ship across other frontier families~\cite{claude5window}.

Per-turn latency understates what a repair core delivers, so we also
keep two throughput accounts, strictly apart. The \emph{served}
account credits a turn with all context bytes it delivers, the same
crediting a cache gets. The \emph{scanned} account counts only
bytes physically retokenized, and for every full retokenizer the
two coincide. Trace-weighted over the replay, one repair core
serves 0.44\,GB/s of context while scanning 3\,MB/s, and its served
curve crosses Gigatoken's prewarmed ceiling at ${\sim}200$\,K
context bytes (Figure~\ref{fig:headline}c and
Appendix~\ref{app:equiv}). At the 3\,M-character
shape, the served account reaches 1.4\,GB/s for one core
(Figure~\ref{fig:headline}c). The served account is an
incremental-repair account only. GPU full-tokenization throughput
(\S\ref{sec:eval-cold}) counts physically scanned bytes, and the
two accountings are never mixed.

\subsection{Full tokenization: latency and throughput}
\label{sec:eval-cold}

The complete GPU path sustains 3.8--4.7\,GB/s across English, CJK,
and templated corpora, and pre-tokenization alone reaches
30--77\,GB/s. The strongest CPU configuration we could construct,
fastokens with one single-threaded worker per core and NUMA-local
data placement, sustains 1.35--2.06\,GB/s on the full 32-core host,
and Gigatoken batch mode reaches 0.69--0.85\,GB/s on the same host
on an empty cache.
The unpinned deployment reaches only 0.89--1.26\,GB/s and degrades
beyond 16 processes, so NUMA placement must be reported for CPU
baselines, and we apply the same correction to every CPU number in
this paper.

\begin{figure}[t]
  \centering
  \includegraphics[width=.80\columnwidth]{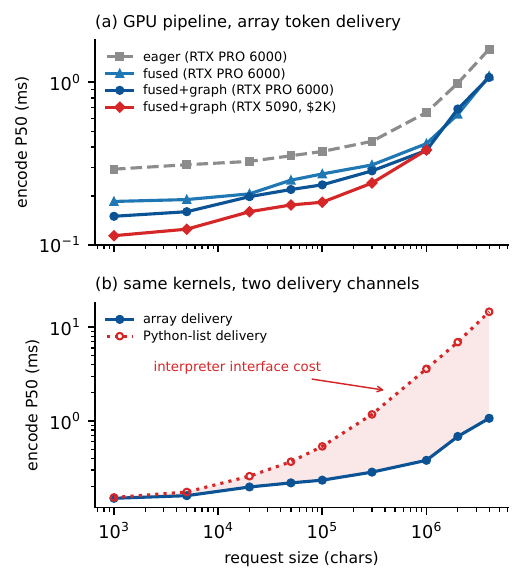}
  \caption{Full tokenization single-request encode P50 (Qwen3 family, steady
  state). (a)~With array token delivery, every dispatch variant
  stays near or below one millisecond across three decades of
  request size, and the consumer RTX~5090 leads the server card.
  (b)~The same fused+graph kernels behind two delivery channels.
  Materializing a Python \texttt{list[int]} adds interpreter
  interface cost that dominates beyond 100\,K characters. Points
  beyond 1\,M characters are extension points from the same
  protocol.}
  \label{fig:latency}
\end{figure}

Session-initialization latency matters more to the service.
Figure~\ref{fig:latency} shows steady-state single-request
measurements. The CUDA-graph path with array output takes 0.15\,ms
at 1\,K characters, 0.38\,ms at 1\,M, 0.68\,ms at 2\,M, and
1.07\,ms at 4\,M. Materializing a Python \texttt{list[int]} adds a
host-side cost that dominates beyond 100\,K characters. We report
array and list delivery separately, because one measures the GPU
pipeline and the other also measures the interpreter interface.
Engine handoff through today's Python APIs pays the list cost, and
our TTFT results include it in full.

\begin{table}[t]
  \centering\small
  \caption{Full tokenization single-request latency (ms) on identical real
  texts, one timed call per sample ($n{=}20$ per cell), context
  sizes in characters. The upper block is the normalized family
  (Qwen3, NFC), the last row the non-normalized family
  (Llama~3.1). Gigatoken uses a new object with an empty cache
  (construction excluded). The LoPT paper reports 116.8\,ms on a
  112-core node at LongBenchV2 lengths, an external anchor. The
  28-process row is the same-host comparison (zero retries at every
  shape). The 4.4\,M shape exceeds the graph path's largest capture
  bucket ($2^{22}$ bytes), so the GPU rows there run the same
  kernels without graph replay.}
  \label{tab:coldhead}
  \setlength{\tabcolsep}{2.2pt}%
  \begin{tabular}{@{}lrrrrr@{}}
    \toprule
    & \multicolumn{4}{c}{P50 by context size} & P90 \\
    \cmidrule(lr){2-5}
    & 100\,K & 1\,M & 2\,M & 4.4\,M & 4.4\,M \\
    \midrule
    \sys{} GPU, array delivery & \textbf{0.29} & \textbf{0.87} & \textbf{1.34} & \textbf{3.59} & 133.1 \\
    \sys{} GPU, Python list & 0.56 & 3.86 & 8.61 & 19.87 & 144.5 \\
    Gigatoken, empty cache & 0.64 & 6.05 & 8.78 & 19.42 & 22.3 \\
    fastokens, 1 core & 2.03 & 20.5 & 45.9 & 111.9 & 120.0 \\
    LoPT repro., 28 processes & 7.91 & 42.2 & 95.3 & 198.5 & 228.4 \\
    HF serial, 1 core & 24.8 & 360.3 & 762.9 & 1762.8 & 1833.5 \\
    \midrule
    \sys{} GPU, array (Llama~3.1) & 0.27 & 0.80 & 1.31 & 3.28 & 4.1 \\
    \bottomrule
  \end{tabular}
\end{table}

Table~\ref{tab:coldhead} uses the stricter serving protocol. Every
request is a new real-text sample, encoded once. At 1\,M
characters, the GPU path answers in 0.87\,ms with array output.
That is \ratioColdSota{} below the fastest previously published
CPU configuration under the same protocol (8.2$\times$ at P90) and
6.9$\times$ below the strongest CPU number we measured from an empty
cache, Gigatoken. At 2\,M characters the respective
values are 1.34\,ms, 45.9\,ms, and 8.78\,ms. Python-list delivery
clears every CPU row through 1\,M characters (3.86 vs.\ 6.05\,ms),
then approaches Gigatoken at the largest sizes. Past 2\,M the
interpreter, not the GPU, binds that channel. Fresh samples also
pay normalization quick-checks and buffer-geometry variance that a
re-encoded corpus slice never sees, which is why this table's 1\,M
median (0.87\,ms) exceeds the steady-state 0.38\,ms. Baseline
comparisons use the sampled protocol only.

The largest Qwen3 samples expose a tail limitation. At 4.4\,M
characters, 4 of 20 inputs fail the GPU NFC quick check and pay CPU
renormalization, lifting P90 to 133\,ms. Llama~3.1 ships no
normalizer and holds P90 at 4.1\,ms on the same shape. We therefore
report normalized and non-normalized families separately in
Table~\ref{tab:coldhead}, and full GPU normalization remains an
implementation gap.

With 28 competing CPU processes, the eager pipeline's P99 doubles
from 0.52 to 1.05\,ms at 50\,K characters, while the graph path
moves from 0.38 to 0.49\,ms (Figure~\ref{fig:contention} in
Appendix~\ref{app:tables}). The kernels are unchanged.

\subsection{Burst tails and capacity}
\label{sec:eval-tail}

\begin{figure*}[t]
  \centering
  \includegraphics[width=.60\textwidth]{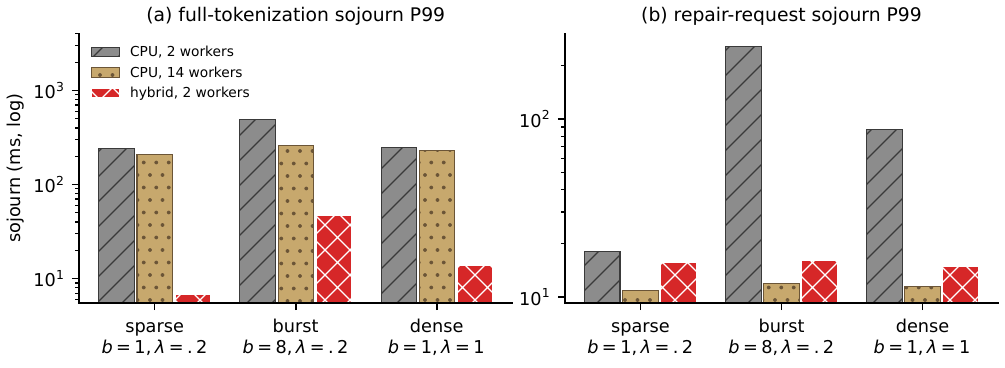}
  \caption{Tail behavior in recorded burst scenarios. Adding CPU
  workers does not remove the session-initialization service-time floor~(a),
  and at low worker counts initialization work drags continuation
  P99 up~(b). A four-core \sys{} repair pool plus one GPU keeps continuation
  and initialization P99 low.
  The host-contention companion panel is
  Figure~\ref{fig:contention} in Appendix~\ref{app:tables}.}
  \label{fig:tail}
\end{figure*}

Single-request medians do not show queueing behavior, so we replay
16 open-loop burst scenarios built from measured session shapes. A CPU-only front end needs seven times as many cores to
keep continuation P99 flat, while initialization P99 stays between
210 and 490\,ms at every tested core count. A four-core \sys{} repair pool
plus one GPU keeps continuation P99 near 15\,ms and initialization
P99 between 6.9 and 46.6\,ms (Figure~\ref{fig:tail}). More CPU workers move the
queueing knee without shortening one large full-context encode.

\begin{figure}[t]
  \centering
  \includegraphics[width=.78\columnwidth]{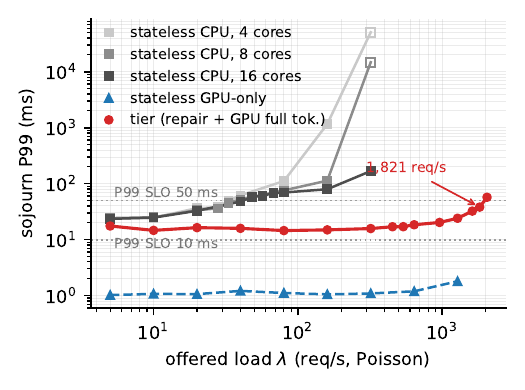}
  \caption{P99 sojourn time versus offered load (Poisson, measured
  mixture, 60\,s steady state per point, and open markers are
  backlogged points reported as-is). The tier reaches
  1{,}821\,requests/s under a 50\,ms P99 objective, while stateless
  CPU configurations saturate at 33--40\,requests/s and only the
  GPU-only front end holds the 10\,ms objective.}
  \label{fig:capacity}
\end{figure}

Figure~\ref{fig:capacity} sweeps Poisson offered load over the
measured request mixture. Session initializations account for 2.3\%
of requests and are sampled from the recorded initialization pool. A stateless CPU front
end cannot meet a 10\,ms P99 objective at any tested core count,
because the $O(N)$ service floor alone exceeds the target. Under a
50\,ms objective, 4, 8, and 16 cores sustain 33, 33, and
40\,requests/s. The tier sustains 1{,}821\,requests/s with four
repair cores and one GPU, more than 45$\times$ the 16-core
stateless capacity. The 45$\times$ ratio compares different
hardware resources. It establishes the capacity of the tested
configurations, and it does not isolate the benefit of session
state from the benefit of the GPU, so we treat it as a
system-capacity result rather than an equal-resource efficiency
result. The bottleneck at 1{,}821\,requests/s is the repair pool, not the
GPU, whose full-tokenization P99 still reads 1.5\,ms there.

A GPU-only stateless front end meets the 10\,ms objective (P99
1.8\,ms at 1{,}280\,requests/s), but it ships the complete text of
every request to the GPU, about 55$\times$ the character volume the
tier sends. Heavy continuation appends of 30--100\,K characters also hold
incremental-repair service P99 at 13--17\,ms independent of load, above the
10\,ms target before queueing begins. The current router does not
redirect these unusually large appends, and delta-size-aware
routing would remove this tail case. Shadow verification at 5\%
sampling re-checked 9{,}161 requests across this sweep with zero
mismatches, on an otherwise idle, preflight-gated machine.

\subsection{vLLM in the loop}
\label{sec:eval-engine}

We submit either text or tier-produced token IDs to vLLM, paired on
identical content with isolated KV prefixes. All runs in this
section use the shipped Rust session store, a single implementation
generation measured end to end. A prior-generation archive of the
same regimes shows gains of the same sign and similar magnitude,
because the store upgrade only accelerates the tier's own
component, which is 0.4--2\% of TTFT. The gain comes from the
engine side.
Differencing the engine's own metrics isolates a front-end segment
that is neither queueing nor prefill, and the tier shrinks that
segment from 37.6--356\,ms per request to 9.8--48\,ms while queue
and prefill stay level (Figure~\ref{fig:engine-rs}a).

\begin{figure}[t]
  \centering
  \includegraphics[width=.92\columnwidth]{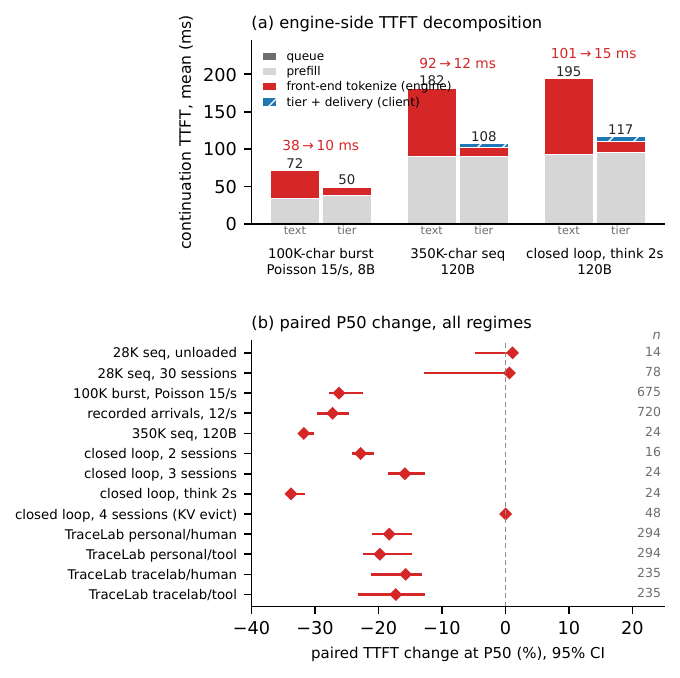}
  \caption{Session-continuation TTFT with vLLM in the loop, text-in vs.\
  tier-in-front on identical paired content. (a)~Engine-side
  decomposition from \texttt{/metrics} differencing shows the gap
  is the front-end tokenization segment, with queue and prefill
  unchanged, and the tier's own client-side work below 6\,ms.
  (b)~Paired P50 change with bootstrap 95\% CIs across all measured
  regimes, where the two-session point uses the control arm with a
  fresh KV cache. Under recorded arrivals P90 reverses ($+$28\%)
  while P50 and P99 improve, a load-generator and batch-shape
  artifact whose decomposition is archived with the run.}
  \label{fig:engine-rs}
\end{figure}

Figure~\ref{fig:engine-rs}b spans six regimes. At 28\,K tokens with
unloaded sequential streams, the two channels tie and the intervals
cross zero, a null result we report deliberately, because
single-stream benchmarks cannot see front-end tokenization. At
100\,K-character contexts with 400-character appends and
15\,requests/s Poisson arrivals, the tier reduces P50 TTFT by 26\%
($n{=}675$). Replaying recorded burst timestamps at a similar mean
rate improves P50 by 27\% and P99 by 23\% (2{,}961 to 2{,}288\,ms), while
P90 reverses by 28\%. The reversal has two archived sources, GIL
contention in the open-loop load generator and a more coherent
burst shape reaching the scheduler, and the engine-measured
end-to-end mean still improves. On gpt-oss-120B with
350\,K-character contexts (95\,K tokens), median TTFT falls by
32\%, with the tier's own component at 0.64\,ms.

Closed-loop experiments preserve causality within each session and
run several sessions concurrently. At 95\,K tokens per session,
median TTFT improves by 23\% with two back-to-back sessions on a
fresh KV cache, 16\% with three, and 34\% with two sessions and a
two-second think time. At four sessions the engine cannot retain
the needed KV state. Prefix reuse collapses, both channels degrade
identically to 66-second full-prefill sojourns within $\pm0.3\%$,
and the tier's own component stays below 1\,ms. Tokenization
placement is irrelevant past the engine's KV capacity.

We also generate closed-loop sessions from TraceLab append and
pacing distributions via the calibrated generator of
\S\ref{sec:workload:sources}. Across 12 paired runs, the median
TTFT improvement is 16--20\%, insensitive to pacing. Compaction
steps land as 0.9--1.3\,s full-prefill requests, while the tier's
state-rebuild component stays below 14\,ms.

\subsection{Cost and resource use}
\label{sec:eval-econ}

Applying the measured per-path costs to the 153{,}529 calls with
complete token accounting gives 13.4\,ms of CPU time per request
for fastokens and 2.3\,ms for \sys{} repair over the same engine.
Here $r$ denotes the request completion rate per inference GPU. At
$r{=}0.5$\,requests/s, those two values correspond to 6.7 and 1.1
front-end CPU cores per 1{,}000 GPUs. GPU full tokenization instead contributes 0.15\,ms of
GPU service time under this mixture. That is router GPU duty rather
than CPU core time, and it normalizes to 0.1 router GPUs per
1{,}000 inference GPUs (Appendix~\ref{app:econ}). We present
this as a rate normalization rather than a cluster design.
Topology, model scale, and batching enter only through $r$.

Direct power measurements show 2.06\,GB/s at 221\,W of CPU package
power for the best 32-core fastokens configuration, and
4.14\,GB/s at 288\,W of GPU board power plus 8\,W of incremental
host power. On the measured corpus the GPU path is about
1.5$\times$ better in tokens per watt, under an accounting that
favors the CPU, since package power excludes DRAM and PSU losses
while the GPU figure is total board power. These figures support
deployment feasibility, while the latency and state-reuse results
remain the primary evidence.

\begin{figure*}[t]
  \centering
  \includegraphics[width=.86\textwidth]{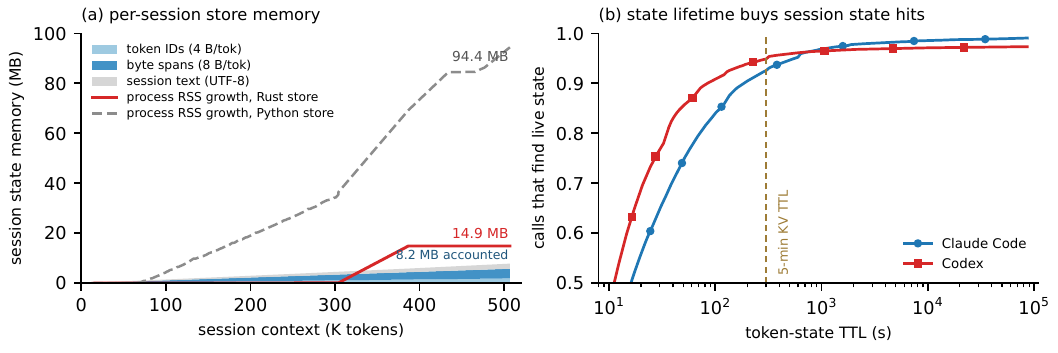}
  \caption{What session state costs and what its lifetime buys.
  (a)~Store memory for one growing session on the replayed real
  trace. Accounted state is ${\sim}16$ bytes per token, the Rust
  store's resident growth stays between 1.2 and 22\,MB per session
  across strains and families, and the prior Python store paid
  70--129\,MB for the same sessions. (b)~Share of the 153{,}951
  trace calls that would find live token state, as a function of
  state retention time. Raising retention from the 5-minute KV
  default to one hour converts 2--6\% of calls from a session state
  miss to a session state hit.}
  \label{fig:memoryttl}
\end{figure*}

Session state is also cheap to keep. Token IDs and byte spans
account for about 16 bytes per token, a growing 500\,K-token
session holds its resident footprint under 15\,MB in the shipped
store, and the state is freed at session end
(Figure~\ref{fig:memoryttl}a). The same figure prices the retention
policy this state makes affordable. In the recorded traces, 5--7\%
of calls arrive more than five minutes after their predecessor and
would miss a KV-lifetime cache, while a token-state TTL of one hour
keeps 97--98\% of calls on a session state hit and a day keeps 97--99\%
(Figure~\ref{fig:memoryttl}b). Holding hours of token state costs
megabytes per session, where the corresponding KV state costs
gigabytes.

\subsection{Runtime verification}
\label{sec:eval-verify}

The shadow verifier earned its place by exposing a
history-dependent bug in a widely deployed production Rust
tokenizer.\footnote{Reported upstream with a deterministic
reproduction.}
During full-scale replay, 39 of 40 sampled sessions diverged from
that engine while every fresh-process differential suite stayed green.
A five-implementation adjudication established that the tier and
the frozen reference agreed. The suspect engine returns different
IDs for the same input after one earlier encode of a prefix at
least 4{,}096 characters long, exactly the serving pattern of
re-encoding a conversation after a small append. The wrong output
has the same token count, so even a length check stays silent.

We also inject 434 faults into 2{,}999 session-shaped requests. The
mutations cover ID substitution, deletion, duplication, adjacent
swap, and a same-length tail resplit, the silent shape of the bug
above. At 10\% sampling, the verifier catches all 39 sampled faults,
produces no false positives on 265 sampled clean requests, and
drops no samples. Detection delay follows the expected geometric
distribution. At 5\% sampling, half of faulty campaigns are
detected within 13 faulty requests and 90\% within 42. The operator
can therefore trade verifier cost for detection latency, while the
drop counter exposes overload.

\section{Related Work}
\label{sec:related}

\paragraph{CPU and segmented tokenization.} HuggingFace
tokenizers~\cite{hftokenizers}, tiktoken~\cite{tiktoken},
Gigatoken~\cite{gigatoken2026}, and vLLM's proposed Rust front
end~\cite{vllmrustfrontend} reduce the cost per input byte with
Rust, SIMD, caching, and parallel CPU execution. They remain
full-context methods from the point of view of an agent turn. Even
a content-cache hit must identify cached pieces through the
complete request. \sys{} instead keeps the token sequence of a
session and repairs only the changed region. The repair-window
engine is pluggable, so faster CPU tokenizers can reduce its local
cost once they pass the same correctness checks.
Appendix~\ref{app:designspace} tabulates the design-space position.

LoPT~\cite{lopt2026} divides one request into overlapping chunks
and matches results at chunk seams. Its goal is within-request CPU
parallelism. \sys{} uses a related overlap idea across time,
matching a fresh window against a session's cached sequence, and
its acceptance condition is different. Our clean-room reproduction
found inputs on which the length threshold stated in its paper
admits a boundary outside the premise of its theorem. That result motivated
the family-specific stable boundary check.

Incremental BPE~\cite{incrementalbpe2026} makes the BPE stage
itself incremental. It maintains, for every prefix of a growing
string, the state needed to extend the tokenization in
$O(\log^2 t)$ per byte, and proves equivalence to the reference
merge procedure when the vocabulary is canonical and its dependency
graph admits a topological order. That line of work deliberately
leaves the surrounding stages unchanged, and its authors note that
normalization and regex pre-tokenization remain offline-oriented.
\sys{} works at the layer above. It keeps the tokenizer
implementation stateless, uses token and byte-span records as
session state, and decides each splice with a per-request check
defined over the full pipeline, so it admits families whose merge
tables fall outside that premise and also supports mid-context
edits.

A parallel line studies BPE as a streaming transduction. Mamouras
et al.~\cite{streamingbpe2026} bound the \emph{delay} of a
merge-rule list, the lookahead needed before a token can be
finalized, and obtain a streaming algorithm whose memory is
independent of the input length. The formal account of BPE by
Berglund and van der Merwe~\cite{berglund2023bpe} shows that a
constant lookahead suffices in principle. These results describe
the merge stage in isolation. The condition \sys{} needs also
covers the pre-tokenizer and the normalizer, which is why its
boundary check is family-specific and is certified per request
rather than assumed.

\paragraph{GPU tokenization.} GPUTOK~\cite{gputok2026},
BlockBPE~\cite{blockbpe2025}, and cuDF subword
tokenization~\cite{cudfsubword} show that subword encoding can use
GPU parallelism. Their supported specifications differ from the
GPT-family reference tokenizers studied here. BlockBPE removes regex
pre-tokenization, cuDF targets WordPiece, and GPUTOK checks a
GPT-2-style byte-level pipeline against its own CPU implementation
of that pipeline. None of them reports agreement with the frozen
production tokenizers evaluated here. \sys{} derives a parallel formulation
of the reference rules and compares both final IDs and intermediate
piece boundaries with frozen reference implementations, across six
GPU-certified production configurations (Qwen3, Llama~3.1, gpt-oss,
and the DeepSeek group of V3, V4-flash, and HY3) spanning four
pattern families. To our knowledge, as of this writing it is the
first GPU tokenizer shown byte-identical to its frozen production
reference under this criterion.

\paragraph{Prefix-cached serving.} vLLM~\cite{vllm2023},
RadixAttention~\cite{sglang2024}, Mooncake~\cite{mooncake2025},
DistServe~\cite{distserve2024}, Splitwise~\cite{splitwise2024},
CacheGen~\cite{cachegen2024}, Dynamo~\cite{dynamo2025}, and
continuous batching~\cite{orca2022} optimize model execution after
token IDs exist, and provider prompt caching exposes the same reuse
to users~\cite{promptcaching}. \sys{} addresses the preceding
front-end step. Reuse has begun to appear at that step as well.
Serving front ends now ship in-memory tokenizer prefix caches that
map a text prefix to the token IDs it produced, keyed by a digest
of that prefix~\cite{llmtokenizer2026,dynamotokenizers2026}. These
caches split a request only at special-token boundaries, where the
pieces compose by construction, so an append that follows ordinary
text yields no reuse. They also live inside a single front-end
process, and they discard the prefix text once it is hashed, so a
hit is trusted rather than re-checked. \sys{} instead keeps token
IDs together with their source byte spans as session state, matches
that state against the request text, and accepts a splice wherever
the per-request boundary check of \S\ref{sec:warm} passes. It uses
token-ID request interfaces that serving
stacks already expose, such as vLLM's \texttt{prompt\_token\_ids}
(\S\ref{sec:integration}) and Dynamo's pre-tokenized backend
requests, so it deploys without changing the model
scheduler.

\paragraph{Agent workloads.} Public inference
traces~\cite{burstgpt2025,mooncake2025} report request sizes and
arrival rates but generally not the relation between a session's
previous context and its new text. TraceLab~\cite{tracelab2026} and
CacheWise~\cite{cachewise2026} characterize coding-agent sessions
for KV-cache management, and their measurements independently
support the long-context, small-append, high-hit-rate pattern used
in this paper. TraceLab discards text, so it cannot replay
tokenization. Our traces retain counts only for privacy, while the
public autonomous-agent corpus~\cite{swebenchpro_traces} releases
message text at preserved lengths, with filler substituted for
redacted spans. CPU interference
studies~\cite{cpuslowdown2026} corroborate the contention regime of
\S\ref{sec:eval-tail}.

\paragraph{Differential validation.} Differential testing has a
long history in compilers~\cite{yang2011csmith}, and translation
validation checks a produced result against a specification without
proving the entire implementation~\cite{pnueli1998tv,necula2000tv}.
\sys{} applies this division to tokenization. A theorem justifies
an accepted repair splice in a family-level model, version-pinned
differential campaigns check the frozen implementation, and runtime
sampling covers history-dependent failures. The paper does not
claim formal verification of the complete tokenizer stack
(Appendix~\ref{app:guarantee}).

\section{Limitations}
\label{sec:discussion}

We collect the system's limitations here, from the physical floor
of the GPU path to the boundary of what zero divergence
establishes.

\paragraph{What limits the GPU path.}
\label{sec:eval-limits}
The GPU tokenizer uses less than 80\,GB/s of DRAM traffic, about
4\% of the card's bandwidth, at low SM occupancy. Raising the
RTX~PRO~6000 power limit from 450 to 600\,W changes no result. The
limiting resource is the dependency chain of exact BPE merges. Each
merge is a dependent lookup whose result decides the next, so the
floor is chain depth times memory round-trip, a latency roofline.

Three measurements place the implementation at that floor. First,
measurements across five GPUs spanning three architecture
generations follow sustained clock more closely than SM count,
bandwidth, or price, and the \$2{,}000 consumer RTX~5090 runs
11--17\% faster than the server-class card
(Appendix~\ref{app:crossgen}). Second,
optimizations that pay on throughput-bound kernels measurably do
not pay here. L2 residency pinning loses 11--12\%,
per-architecture launch retuning moves results within $\pm1$\%, and
an on-GPU piece-memoization prototype with a 96.8\% hit rate
returns $+0.8$\%. Third, a gated prototype that commits rank
plateaus in parallel gains 42\% CJK throughput but diverges from
the reference on 4.0\% of Qwen pieces and 0.34\% of Llama pieces.
We do not use it. The remaining chain is the price of exactness,
and it implies a router-resident tier wants one cheap high-clock
card rather than a datacenter GPU.

\paragraph{Coverage.} Incremental repair requires a supported family and a
stable boundary inside the matched region. Two of the 17 families
we examined do not satisfy the current predicate and always use
full tokenization. WordPiece follows a different boundary argument
and remains on the CPU path in the current system. Added-token extraction and one NFC normalizer
also stay on the CPU. These fallbacks preserve output correctness
and shrink the fraction of traffic that receives the fast path.

\paragraph{Session-state construction.} The GPU encoder currently
emits token IDs without source byte spans. A request that
initializes session state therefore runs the reference tokenizer to
obtain spans. In the measured engine-in-loop runs this adds at most
14\,ms to a request whose model prefill takes 0.9--1.3\,s.
Exporting spans on the GPU is the highest-leverage missing piece of
engineering. Cross-worker state transfer is the other planned
extension. Today a request that loses worker affinity becomes a
session state miss and stays correct (\S\ref{sec:state}).

\paragraph{Workload scope.} The personal trace panel covers six
users and nine machines and is concentrated on coding agents.
Provider aggregates, a public autonomous-agent trace, and TraceLab
reduce the risk that one client or group determines the result, but
they do not cover all agent applications. The benefit also depends
on context length, since a fully prewarmed content cache wins on
quiet cores at 100\,K characters and below (\S\ref{sec:eval-warm})
and \sys{} grows stronger as the context grows relative to the
append.

\paragraph{Validation boundary.} Zero divergence in the reported
campaigns is evidence for the tested artifacts, not a proof about
future tokenizer versions. Each new snapshot must repeat family
admission and differential validation. The shadow verifier stays
deployed, because one observed failure depends on prior request
history and cannot be ruled out by fresh-process testing alone. All
personal traces were collected with consent under a count-only
discipline (Appendix~\ref{app:workload}), and the
history-dependent-tokenizer bug was reported upstream with a
deterministic reproduction before publication.

\section{Conclusion}
\label{sec:conclusion}

Agentic serving repeatedly submits long contexts after small
updates. KV prefix caching exploits that structure only after the
front end has produced token IDs, and the tokenizer caches that
front ends ship today reuse earlier work only inside one process
and only up to the last special-token boundary. \sys{} carries
the same stateful view into tokenization. It repairs continuing sessions
at a checked boundary, sends large session initializations through an exact
GPU path, and falls back to the reference implementation whenever a
fast-path condition is unavailable. The result is a front end whose
common-case work follows the change in the request, with output
compatible with existing models and prefix caches.

\paragraph{Deployment.} Our evaluation places one GPU at the
serving entry point and uses session affinity across repair
workers. The tokenizer consumes little GPU bandwidth, so colocation
with an inference GPU is possible, and the cross-generation results
suggest one high-clock consumer card suffices
(\S\ref{sec:eval-limits}). A production deployment would also need
state replication, admission control, and a state-lifetime policy,
none of which changes the tokenization algorithms, though each can
affect the session state hit rate.

\paragraph{Delta-aware routing.} Continuation appends beyond ${\sim}30$\,K
characters hold incremental-repair service P99 above a 10\,ms objective
(\S\ref{sec:eval-tail}), and the shipped router does not redirect
them. Head-to-head measurements show where each path is optimal.
Repair cost follows the append and is flat in context, while a GPU
rebuild of the complete context costs 1.1--3.2\,ms up to window-max
scale, so rerouting appends of 50\,K characters and above would cut
their service time by 6--28$\times$ relative to the shipped repair
path. Figure~\ref{fig:routing-phase} in Appendix~\ref{app:routes}
maps the measured routing regions, and the three explored routes
are detailed there. Only 0.8\% of measured
appends are that large, and these numbers use exploration-prototype
accounting, so the shipped default remains repair for every session
continuation until the rerouted path passes the same certification
battery.

\bibliographystyle{ACM-Reference-Format}
\bibliography{refs}


\begin{thebibliography}{40}

%%% the Web address.

\ifx \showCODEN    \undefined \def \showCODEN     #1{\unskip}     \fi
\ifx \showDOI      \undefined \def \showDOI       #1{#1}\fi
\ifx \showISBNx    \undefined \def \showISBNx     #1{\unskip}     \fi
\ifx \showISBNxiii \undefined \def \showISBNxiii  #1{\unskip}     \fi
\ifx \showISSN     \undefined \def \showISSN      #1{\unskip}     \fi
\ifx \showLCCN     \undefined \def \showLCCN      #1{\unskip}     \fi
\ifx \shownote     \undefined \def \shownote      #1{#1}          \fi
\ifx \showarticletitle \undefined \def \showarticletitle #1{#1}   \fi
\ifx \showURL      \undefined \def \showURL       {\relax}        \fi
% invisible to TeX
\providecommand\bibfield[2]{#2}
\providecommand\bibinfo[2]{#2}
\providecommand\natexlab[1]{#1}
\providecommand\showeprint[2][]{arXiv:#2}

\bibitem[{Anthropic}(2024)]%
        {promptcaching}
\bibfield{author}{\bibinfo{person}{{Anthropic}}.}
  \bibinfo{year}{2024}\natexlab{}.
\newblock \bibinfo{title}{Prompt Caching with {Claude}}.
\newblock
\newblock
\newblock
\shownote{Explicit 5-minute and 1-hour cache TTLs}.


\bibitem[{Anthropic}(2026)]%
        {claude5window}
\bibfield{author}{\bibinfo{person}{{Anthropic}}.}
  \bibinfo{year}{2026}\natexlab{}.
\newblock \bibinfo{title}{Models overview}.
\newblock
  \bibinfo{howpublished}{\url{https://platform.claude.com/docs/en/about-claude/models/overview}}.
\newblock
\newblock
\shownote{Claude Fable 5 and Claude Opus 5: 1M token context window. Accessed
  July 30, 2026}.


\bibitem[Berglund and van~der Merwe(2023)]%
        {berglund2023bpe}
\bibfield{author}{\bibinfo{person}{Martin Berglund} {and}
  \bibinfo{person}{Brink van~der Merwe}.} \bibinfo{year}{2023}\natexlab{}.
\newblock \bibinfo{title}{Formalizing {BPE} Tokenization}.
\newblock
\newblock
\newblock
\shownote{arXiv:2309.08715}.


\bibitem[Chung et~al\mbox{.}(2026)]%
        {cpuslowdown2026}
\bibfield{author}{\bibinfo{person}{Euijun Chung}, \bibinfo{person}{Yuxiao Jia},
  \bibinfo{person}{Aaron Jezghani}, {and} \bibinfo{person}{Hyesoon Kim}.}
  \bibinfo{year}{2026}\natexlab{}.
\newblock \bibinfo{title}{Characterizing {CPU}-Induced Slowdowns in Multi-{GPU}
  {LLM} Inference}.
\newblock
\newblock
\newblock
\shownote{arXiv:2603.22774}.


\bibitem[{DeepSeek-AI}(2024a)]%
        {deepseekv3}
\bibfield{author}{\bibinfo{person}{{DeepSeek-AI}}.}
  \bibinfo{year}{2024}\natexlab{a}.
\newblock \bibinfo{title}{{DeepSeek-V3} Technical Report}.
\newblock
\newblock
\newblock
\shownote{arXiv:2412.19437}.


\bibitem[{DeepSeek-AI}(2024b)]%
        {deepseekv3tok}
\bibfield{author}{\bibinfo{person}{{DeepSeek-AI}}.}
  \bibinfo{year}{2024}\natexlab{b}.
\newblock \bibinfo{title}{{DeepSeek-V3} tokenizer.json}.
\newblock
\newblock
\newblock
\shownote{HuggingFace model artifact, revision e815299b.
  \url{https://huggingface.co/deepseek-ai/DeepSeek-V3}}.


\bibitem[Deng et~al\mbox{.}(2025)]%
        {swebenchpro2025}
\bibfield{author}{\bibinfo{person}{Xiang Deng}, \bibinfo{person}{Jeff Da},
  \bibinfo{person}{Edwin Pan}, {et~al\mbox{.}}}
  \bibinfo{year}{2025}\natexlab{}.
\newblock \bibinfo{title}{{SWE-Bench Pro}: Can {AI} Agents Solve Long-Horizon
  Software Engineering Tasks?}
\newblock
\newblock
\newblock
\shownote{arXiv:2509.16941}.


\bibitem[{fastokens contributors}(2026)]%
        {fastokens}
\bibfield{author}{\bibinfo{person}{{fastokens contributors}}.}
  \bibinfo{year}{2026}\natexlab{}.
\newblock \bibinfo{title}{fastokens: A Fast {BPE} Tokenizer with a {Rust}
  Backend}.
\newblock
\newblock
\newblock
\shownote{Version 0.2.0, Apache-2.0. Supported by vLLM v0.23.0 and later as an
  opt-in tokenizer backend that must be installed separately.
  \url{https://github.com/crusoecloud/fastokens}}.


\bibitem[{HuggingFace}(2019)]%
        {hftokenizers}
\bibfield{author}{\bibinfo{person}{{HuggingFace}}.}
  \bibinfo{year}{2019}\natexlab{}.
\newblock \bibinfo{title}{HuggingFace Tokenizers}.
\newblock
\newblock
\newblock
\shownote{Rust library; version 0.22.2 in the frozen environment of this paper.
  \url{https://github.com/huggingface/tokenizers}}.


\bibitem[{Inferact}(2026)]%
        {swebenchpro_traces}
\bibfield{author}{\bibinfo{person}{{Inferact}}.}
  \bibinfo{year}{2026}\natexlab{}.
\newblock \bibinfo{title}{codex\_swebenchpro\_traces: Agentic Workload Traces
  of {Codex} on {SWE-Bench Pro}}.
\newblock
\newblock
\newblock
\shownote{HuggingFace dataset, MIT license. Redacted spans are replaced by
  length-preserving filler text. File codex\_swebenchpro.json, SHA-256 prefix
  670f1ae8325fd70a, retrieved July 13, 2026}.


\bibitem[Jawa(2021)]%
        {cudfsubword}
\bibfield{author}{\bibinfo{person}{Vibhu Jawa}.}
  \bibinfo{year}{2021}\natexlab{}.
\newblock \bibinfo{title}{Run State of the Art {NLP} Workloads at Scale with
  {RAPIDS}, {HuggingFace}, and {Dask}}.
\newblock
\newblock
\newblock
\shownote{NVIDIA Developer Blog; {BERT} WordPiece batch tokenization. Accessed
  July 31, 2026.
  \url{https://developer.nvidia.com/blog/run-state-of-the-art-nlp-workloads-at-scale-with-rapids-huggingface-and-dask/}}.


\bibitem[Jiang and Gong(2026)]%
        {incrementalbpe2026}
\bibfield{author}{\bibinfo{person}{Shenghu Jiang} {and} \bibinfo{person}{Ruihao
  Gong}.} \bibinfo{year}{2026}\natexlab{}.
\newblock \showarticletitle{Incremental BPE Tokenization}. In
  \bibinfo{booktitle}{\emph{Proceedings of ICML}}.
\newblock
\newblock
\shownote{arXiv:2605.30813}.


\bibitem[Kadamba and Jaisankar(2026)]%
        {gputok2026}
\bibfield{author}{\bibinfo{person}{Venu~Gopal Kadamba} {and}
  \bibinfo{person}{Kanishkha Jaisankar}.} \bibinfo{year}{2026}\natexlab{}.
\newblock \bibinfo{title}{{GPUTOK}: {GPU} Accelerated Byte Level {BPE}
  Tokenization}.
\newblock
\newblock
\newblock
\shownote{arXiv:2603.02597}.


\bibitem[Kwon et~al\mbox{.}(2023)]%
        {vllm2023}
\bibfield{author}{\bibinfo{person}{Woosuk Kwon}, \bibinfo{person}{Zhuohan Li},
  \bibinfo{person}{Siyuan Zhuang}, \bibinfo{person}{Ying Sheng},
  \bibinfo{person}{Lianmin Zheng}, \bibinfo{person}{Cody~Hao Yu},
  \bibinfo{person}{Joseph Gonzalez}, \bibinfo{person}{Hao Zhang}, {and}
  \bibinfo{person}{Ion Stoica}.} \bibinfo{year}{2023}\natexlab{}.
\newblock \showarticletitle{Efficient Memory Management for Large Language
  Model Serving with {PagedAttention}}. In
  \bibinfo{booktitle}{\emph{Proceedings of SOSP}}.
\newblock


\bibitem[Liu et~al\mbox{.}(2024)]%
        {cachegen2024}
\bibfield{author}{\bibinfo{person}{Yuhan Liu}, \bibinfo{person}{Hanchen Li},
  \bibinfo{person}{Yihua Cheng}, {et~al\mbox{.}}}
  \bibinfo{year}{2024}\natexlab{}.
\newblock \showarticletitle{{CacheGen}: {KV} Cache Compression and Streaming
  for Fast Large Language Model Serving}. In
  \bibinfo{booktitle}{\emph{Proceedings of SIGCOMM}}.
\newblock


\bibitem[Mamouras et~al\mbox{.}(2026)]%
        {streamingbpe2026}
\bibfield{author}{\bibinfo{person}{Konstantinos Mamouras},
  \bibinfo{person}{Angela~W. Li}, {and} \bibinfo{person}{Yudi Yang}.}
  \bibinfo{year}{2026}\natexlab{}.
\newblock \showarticletitle{An Efficient Algorithm for Streaming {BPE}
  Tokenization}.
\newblock \bibinfo{journal}{\emph{Proceedings of the ACM on Programming
  Languages}} \bibinfo{volume}{10}, \bibinfo{number}{PLDI}, Article
  \bibinfo{articleno}{252} (\bibinfo{year}{2026}),
  \bibinfo{numpages}{2085--2108}~pages.
\newblock
\urldef\tempurl%
\url{https://doi.org/10.1145/3808330}
\showDOI{\tempurl}


\bibitem[Necula(2000)]%
        {necula2000tv}
\bibfield{author}{\bibinfo{person}{George~C. Necula}.}
  \bibinfo{year}{2000}\natexlab{}.
\newblock \showarticletitle{Translation Validation for an Optimizing Compiler}.
  In \bibinfo{booktitle}{\emph{Proceedings of the ACM SIGPLAN Conference on
  Programming Language Design and Implementation (PLDI)}}.
\newblock


\bibitem[{NVIDIA}(2024)]%
        {nvidia_blackwell}
\bibfield{author}{\bibinfo{person}{{NVIDIA}}.} \bibinfo{year}{2024}\natexlab{}.
\newblock \bibinfo{title}{{NVIDIA} {Blackwell} Platform: {GB200} {NVL72}}.
\newblock
\newblock
\newblock
\shownote{Vendor-reported claim: up to 30$\times$ LLM-inference throughput vs.\
  the same number of {H100} GPUs; named mechanisms are FP4 in a
  second-generation Transformer Engine and fifth-generation NVLink. Accessed
  July 31, 2026.
  \url{https://nvidianews.nvidia.com/news/nvidia-blackwell-platform-arrives-to-power-a-new-era-of-computing}}.


\bibitem[{NVIDIA}(2025)]%
        {dynamo2025}
\bibfield{author}{\bibinfo{person}{{NVIDIA}}.} \bibinfo{year}{2025}\natexlab{}.
\newblock \bibinfo{title}{{NVIDIA} Dynamo: A Datacenter-Scale Distributed
  Inference Serving Framework}.
\newblock
\newblock
\newblock
\shownote{Router documentation: backend handlers receive pre-tokenized
  requests}.


\bibitem[{NVIDIA}(2026)]%
        {dynamotokenizers2026}
\bibfield{author}{\bibinfo{person}{{NVIDIA}}.} \bibinfo{year}{2026}\natexlab{}.
\newblock \bibinfo{title}{{dynamo-tokenizers}: Tokenizer Backends and Prefix
  Cache for the {NVIDIA} Dynamo Front End}.
\newblock
\newblock
\newblock
\shownote{Rust crate. In-memory prefix cache keyed by a digest of the text
  prefix and split at special-token boundaries; enabled by default since June
  2026. \url{https://github.com/ai-dynamo/frontend-crates}}.


\bibitem[{OpenAI}(2023)]%
        {tiktoken}
\bibfield{author}{\bibinfo{person}{{OpenAI}}.} \bibinfo{year}{2023}\natexlab{}.
\newblock \bibinfo{title}{tiktoken: A Fast {BPE} Tokeniser for Use with
  {OpenAI}'s Models}.
\newblock
\newblock
\newblock
\shownote{\url{https://github.com/openai/tiktoken}}.


\bibitem[{OpenAI}(2026)]%
        {gpt56sol}
\bibfield{author}{\bibinfo{person}{{OpenAI}}.} \bibinfo{year}{2026}\natexlab{}.
\newblock \bibinfo{title}{{GPT-5.6 Sol Model} | {OpenAI API}}.
\newblock
  \bibinfo{howpublished}{\url{https://developers.openai.com/api/docs/models/gpt-5.6-sol}}.
\newblock
\newblock
\shownote{Context window: 1,050,000 tokens. Accessed July 30, 2026}.


\bibitem[Patel et~al\mbox{.}(2024)]%
        {splitwise2024}
\bibfield{author}{\bibinfo{person}{Pratyush Patel}, \bibinfo{person}{Esha
  Choukse}, \bibinfo{person}{Chaojie Zhang}, {et~al\mbox{.}}}
  \bibinfo{year}{2024}\natexlab{}.
\newblock \showarticletitle{Splitwise: Efficient Generative {LLM} Inference
  Using Phase Splitting}. In \bibinfo{booktitle}{\emph{Proceedings of ISCA}}.
\newblock
\newblock
\shownote{Azure LLM inference traces}.


\bibitem[Pnueli et~al\mbox{.}(1998)]%
        {pnueli1998tv}
\bibfield{author}{\bibinfo{person}{Amir Pnueli}, \bibinfo{person}{Michael
  Siegel}, {and} \bibinfo{person}{Eli Singerman}.}
  \bibinfo{year}{1998}\natexlab{}.
\newblock \showarticletitle{Translation Validation}. In
  \bibinfo{booktitle}{\emph{Tools and Algorithms for the Construction and
  Analysis of Systems (TACAS)}}.
\newblock


\bibitem[Qin et~al\mbox{.}(2025)]%
        {mooncake2025}
\bibfield{author}{\bibinfo{person}{Ruoyu Qin}, \bibinfo{person}{Zheming Li},
  \bibinfo{person}{Weiran He}, \bibinfo{person}{Jialei Cui},
  \bibinfo{person}{Feng Ren}, \bibinfo{person}{Mingxing Zhang},
  \bibinfo{person}{Yongwei Wu}, \bibinfo{person}{Weimin Zheng}, {and}
  \bibinfo{person}{Xinran Xu}.} \bibinfo{year}{2025}\natexlab{}.
\newblock \showarticletitle{Mooncake: Trading More Storage for Less Computation
  --- A {KVCache}-Centric Architecture for Serving {LLM} Chatbot}. In
  \bibinfo{booktitle}{\emph{Proceedings of the 23rd USENIX Conference on File
  and Storage Technologies (FAST)}}. \bibinfo{publisher}{USENIX Association},
  \bibinfo{pages}{155--170}.
\newblock


\bibitem[Radford et~al\mbox{.}(2019)]%
        {gpt2}
\bibfield{author}{\bibinfo{person}{Alec Radford}, \bibinfo{person}{Jeffrey Wu},
  \bibinfo{person}{Rewon Child}, \bibinfo{person}{David Luan},
  \bibinfo{person}{Dario Amodei}, {and} \bibinfo{person}{Ilya Sutskever}.}
  \bibinfo{year}{2019}\natexlab{}.
\newblock \bibinfo{title}{Language Models are Unsupervised Multitask Learners}.
\newblock
\newblock
\newblock
\shownote{GPT-2; byte-level BPE with regex pre-tokenization.
  \url{https://github.com/openai/gpt-2}}.


\bibitem[R{\o}d(2026)]%
        {gigatoken2026}
\bibfield{author}{\bibinfo{person}{Marcel R{\o}d}.}
  \bibinfo{year}{2026}\natexlab{}.
\newblock \bibinfo{title}{Gigatoken: {SIMD} and Cache Hierarchies for 1000x
  Faster Byte-Pair Encoding Tokenization on Modern {CPUs}}.
\newblock
\newblock
\newblock
\shownote{Audited at version 0.9.0, commit 0d9765fa.
  \url{https://github.com/marcelroed/gigatoken}}.


\bibitem[Shao et~al\mbox{.}(2026)]%
        {lopt2026}
\bibfield{author}{\bibinfo{person}{Wei Shao}, \bibinfo{person}{Lingchao Zheng},
  \bibinfo{person}{Pengyu Wang}, \bibinfo{person}{Peizhen Zheng},
  \bibinfo{person}{Jun Li}, {and} \bibinfo{person}{Yuwei Fan}.}
  \bibinfo{year}{2026}\natexlab{}.
\newblock \showarticletitle{LoPT: Lossless Parallel Tokenization Acceleration
  for Long Context Inference of Large Language Model}. In
  \bibinfo{booktitle}{\emph{Proceedings of the 64th Annual Meeting of the
  Association for Computational Linguistics (Volume 1: Long Papers)}}.
  \bibinfo{publisher}{Association for Computational Linguistics},
  \bibinfo{address}{San Diego, California, United States},
  \bibinfo{pages}{33107--33122}.
\newblock
\urldef\tempurl%
\url{https://doi.org/10.18653/v1/2026.acl-long.1529}
\showDOI{\tempurl}


\bibitem[{The llm-tokenizer authors}(2026)]%
        {llmtokenizer2026}
\bibfield{author}{\bibinfo{person}{{The llm-tokenizer authors}}.}
  \bibinfo{year}{2026}\natexlab{}.
\newblock \bibinfo{title}{{llm-tokenizer}: A Caching Tokenizer Layer for {LLM}
  Gateways}.
\newblock
\newblock
\newblock
\shownote{Rust crate, first released January 2026. In-memory caching of
  tokenization results. \url{https://github.com/lightseekorg/smg}}.


\bibitem[Tiwari et~al\mbox{.}(2026)]%
        {cachewise2026}
\bibfield{author}{\bibinfo{person}{Shubham Tiwari}, \bibinfo{person}{Tapan
  Chugh}, \bibinfo{person}{Nash Rickert}, \bibinfo{person}{Simon Peter},
  \bibinfo{person}{Ratul Mahajan}, {and} \bibinfo{person}{Haiying Shen}.}
  \bibinfo{year}{2026}\natexlab{}.
\newblock \bibinfo{title}{{CacheWise}: Understanding Workloads and Optimizing
  {KVCache} Management for Efficiently Serving {LLM} Coding Agents}.
\newblock
\newblock
\newblock
\shownote{arXiv:2606.16824}.


\bibitem[{vLLM contributors}(2026)]%
        {vllmrustfrontend}
\bibfield{author}{\bibinfo{person}{{vLLM contributors}}.}
  \bibinfo{year}{2026}\natexlab{}.
\newblock \bibinfo{title}{{[RFC]}: Rust front-end}.
\newblock
\newblock
\newblock
\shownote{GitHub issue 40846, open; proposed as an experimental preview behind
  an opt-in flag, with the Python front end kept as the default.
  \url{https://github.com/vllm-project/vllm/issues/40846}; parity tracking
  \#44280}.


\bibitem[Wang et~al\mbox{.}(2025)]%
        {burstgpt2025}
\bibfield{author}{\bibinfo{person}{Yuxin Wang}, \bibinfo{person}{Yuhan Chen},
  \bibinfo{person}{Zeyu Li}, \bibinfo{person}{Xueze Kang},
  \bibinfo{person}{Yuchu Fang}, \bibinfo{person}{Yeju Zhou},
  \bibinfo{person}{Yang Zheng}, \bibinfo{person}{Zhenheng Tang},
  \bibinfo{person}{Xin He}, \bibinfo{person}{Rui Guo}, \bibinfo{person}{Xin
  Wang}, \bibinfo{person}{Qiang Wang}, \bibinfo{person}{Amelie~Chi Zhou}, {and}
  \bibinfo{person}{Xiaowen Chu}.} \bibinfo{year}{2025}\natexlab{}.
\newblock \showarticletitle{{BurstGPT}: A Real-World Workload Dataset to
  Optimize {LLM} Serving Systems}. In \bibinfo{booktitle}{\emph{Proceedings of
  the 31st ACM SIGKDD Conference on Knowledge Discovery and Data Mining
  (KDD)}}. \bibinfo{pages}{5831--5841}.
\newblock
\urldef\tempurl%
\url{https://doi.org/10.1145/3711896.3737413}
\showDOI{\tempurl}


\bibitem[Yang et~al\mbox{.}(2024)]%
        {sweagent2024}
\bibfield{author}{\bibinfo{person}{John Yang}, \bibinfo{person}{Carlos~E.
  Jimenez}, \bibinfo{person}{Alexander Wettig}, \bibinfo{person}{Kilian
  Lieret}, \bibinfo{person}{Shunyu Yao}, \bibinfo{person}{Karthik Narasimhan},
  {and} \bibinfo{person}{Ofir Press}.} \bibinfo{year}{2024}\natexlab{}.
\newblock \showarticletitle{{SWE-agent}: Agent-Computer Interfaces Enable
  Automated Software Engineering}. In \bibinfo{booktitle}{\emph{Proceedings of
  NeurIPS}}.
\newblock
\newblock
\shownote{arXiv:2405.15793}.


\bibitem[Yang et~al\mbox{.}(2011)]%
        {yang2011csmith}
\bibfield{author}{\bibinfo{person}{Xuejun Yang}, \bibinfo{person}{Yang Chen},
  \bibinfo{person}{Eric Eide}, {and} \bibinfo{person}{John Regehr}.}
  \bibinfo{year}{2011}\natexlab{}.
\newblock \showarticletitle{Finding and Understanding Bugs in {C} Compilers}.
  In \bibinfo{booktitle}{\emph{Proceedings of the 32nd ACM SIGPLAN Conference
  on Programming Language Design and Implementation (PLDI)}}.
\newblock


\bibitem[Yao et~al\mbox{.}(2023)]%
        {react2023}
\bibfield{author}{\bibinfo{person}{Shunyu Yao}, \bibinfo{person}{Jeffrey Zhao},
  \bibinfo{person}{Dian Yu}, \bibinfo{person}{Nan Du}, \bibinfo{person}{Izhak
  Shafran}, \bibinfo{person}{Karthik Narasimhan}, {and} \bibinfo{person}{Yuan
  Cao}.} \bibinfo{year}{2023}\natexlab{}.
\newblock \showarticletitle{{ReAct}: Synergizing Reasoning and Acting in
  Language Models}. In \bibinfo{booktitle}{\emph{Proceedings of ICLR}}.
\newblock
\newblock
\shownote{arXiv:2210.03629}.


\bibitem[You(2025)]%
        {blockbpe2025}
\bibfield{author}{\bibinfo{person}{Amos You}.} \bibinfo{year}{2025}\natexlab{}.
\newblock \bibinfo{title}{BlockBPE: Parallel BPE Tokenization}.
\newblock
\newblock
\newblock
\shownote{arXiv:2507.11941; ES-FoMo III workshop at ICML 2025}.


\bibitem[Yu et~al\mbox{.}(2022)]%
        {orca2022}
\bibfield{author}{\bibinfo{person}{Gyeong-In Yu}, \bibinfo{person}{Joo~Seong
  Jeong}, \bibinfo{person}{Geon-Woo Kim}, \bibinfo{person}{Soojeong Kim}, {and}
  \bibinfo{person}{Byung-Gon Chun}.} \bibinfo{year}{2022}\natexlab{}.
\newblock \showarticletitle{Orca: A Distributed Serving System for
  Transformer-Based Generative Models}. In
  \bibinfo{booktitle}{\emph{Proceedings of the 16th USENIX Symposium on
  Operating Systems Design and Implementation (OSDI)}}.
  \bibinfo{publisher}{USENIX Association}, \bibinfo{pages}{521--538}.
\newblock


\bibitem[Zheng et~al\mbox{.}(2024)]%
        {sglang2024}
\bibfield{author}{\bibinfo{person}{Lianmin Zheng}, \bibinfo{person}{Liangsheng
  Yin}, \bibinfo{person}{Zhiqiang Xie}, {et~al\mbox{.}}}
  \bibinfo{year}{2024}\natexlab{}.
\newblock \showarticletitle{{SGLang}: Efficient Execution of Structured
  Language Model Programs}. In \bibinfo{booktitle}{\emph{Proceedings of
  NeurIPS}}.
\newblock
\newblock
\shownote{RadixAttention prefix caching}.


\bibitem[Zhong et~al\mbox{.}(2024)]%
        {distserve2024}
\bibfield{author}{\bibinfo{person}{Yinmin Zhong}, \bibinfo{person}{Shengyu
  Liu}, \bibinfo{person}{Junda Chen}, \bibinfo{person}{Jianbo Hu},
  \bibinfo{person}{Yibo Zhu}, \bibinfo{person}{Xuanzhe Liu},
  \bibinfo{person}{Xin Jin}, {and} \bibinfo{person}{Hao Zhang}.}
  \bibinfo{year}{2024}\natexlab{}.
\newblock \showarticletitle{{DistServe}: Disaggregating Prefill and Decoding
  for Goodput-Optimized Large Language Model Serving}. In
  \bibinfo{booktitle}{\emph{Proceedings of the 18th USENIX Symposium on
  Operating Systems Design and Implementation (OSDI)}}.
  \bibinfo{publisher}{USENIX Association}, \bibinfo{pages}{193--210}.
\newblock


\bibitem[Zhu et~al\mbox{.}(2026)]%
        {tracelab2026}
\bibfield{author}{\bibinfo{person}{Kan Zhu}, \bibinfo{person}{Mathew Jacob},
  \bibinfo{person}{Chenxi Ma}, \bibinfo{person}{Yi Pan},
  \bibinfo{person}{Stephanie Wang}, \bibinfo{person}{Arvind Krishnamurthy},
  {and} \bibinfo{person}{Baris Kasikci}.} \bibinfo{year}{2026}\natexlab{}.
\newblock \bibinfo{title}{{TraceLab}: Characterizing Coding Agent Workloads for
  {LLM} Serving}.
\newblock
\newblock
\newblock
\shownote{arXiv:2606.30560; dataset release v0.0.1, CC BY 4.0,
  syfi\_coding\_trace SHA-256 prefix 9d265eae69a31cae.
  \url{https://github.com/uw-syfi/TraceLab}}.


\end{thebibliography}

\appendix
\section{The Splice Theorem}
\label{app:proof}

This appendix states and proves the losslessness theorem behind the
incremental repair's splice certificate (\S\ref{sec:warm}). Full per-family
discharge proofs, the v1$\to$v2 repair history (the digit-grouping
counterexample of \S\ref{sec:warm}), and the adversarial discovery
battery are in the companion document shipped with the artifact. The
per-family results are summarized in \S\ref{app:discovery}.

\subsection{Setting}

\begin{definition}[Token records]\label{d:records}
For a string $S$, the no-post-processing tokenization is an ordered
sequence $F(S) = (e_1,\dots,e_n)$ of \emph{records}
$e_i = (\mathrm{id}_i, a_i, b_i)$ with $[a_i,b_i) \subseteq [0,|S|)$
the source span reported by the offset mapping. Spans may leave gaps
(dropped whitespace), may repeat (several byte-level tokens of one
multi-byte character), and never partition $[0,|S|)$ in general.
Sequence \emph{index} is the only structure the theory uses.
\end{definition}

\begin{definition}[Front end and units]\label{d:frontend}
Write the pipeline as $F = \ME^\ast \circ \FE$:
\begin{itemize}
  \item $\FE$ (\emph{front end}) maps $S$ to an ordered sequence of
    \emph{units} $u_1,\dots,u_m$, each
    $u = (\mathrm{kind}, \mathrm{payload}, I)$ with
    $\mathrm{kind} \in \{\text{piece}, \text{added literal}\}$,
    $\mathrm{payload}$ the normalized text the model stage will
    consume, and $I \subseteq [0,|S|)$ the source interval. Source
    intervals are ordered and non-over\-lap\-ping but need not cover $S$.
    For HF fast tokenizers $\FE$ = added-token extraction, then
    per-segment normalization, then pre-tokenization.
  \item $\ME$ (\emph{model stage}) maps one unit to a
    \emph{token-ID} sequence, as a function of
    $(\mathrm{kind},\mathrm{payload})$ alone: an added literal maps to
    its single ID; a piece maps to the IDs of its WordPiece/BPE
    encoding. $\ME^\ast$ concatenates:
    \[
    \ME^\ast(u_1,\dots,u_m) = \ME(u_1)\cat\cdots\cat\ME(u_m).
    \]
\end{itemize}
The losslessness claim is on the ID sequence
\[
F_{\mathrm{id}}(S) = (\mathrm{id}_1,\dots,\mathrm{id}_n).
\]
The global
spans $a_i,b_i$ of Definition~\ref{d:records} are \emph{not} outputs of
$\ME$ (a payload repeated at two positions has one ID but two spans);
the merge reconstructs each token's span from its unit's source interval
$I$ plus the chunk base, an operation that never crosses a junction and
is therefore outside the homomorphism. Writing $F$ for the ID sequence
below is thus without loss.
\end{definition}

\begin{assumption}[Protocol]\label{p:protocol}
Chunks are encoded with
\tok{add\_\allowbreak special\_\allowbreak tokens=\allowbreak False},
no padding,
no truncation; sequence-level post-processing (special-token
wrapping) is applied once, after the merge. This matches the deployed
merge path.
\end{assumption}

\begin{assumption}[Pipeline fidelity]\label{p:fidelity}
$F$ factors as in Definition~\ref{d:frontend} with $\ME$ per-unit
deterministic and stateless across units, for
\tok{tokenizers==0.22.2} as pinned by the manifest. BPE dropout is
off. This is the load-bearing assumption: the model stage is
per-piece, with no cross-piece state. Our differential campaigns
support it for the reference implementation; the history-dependent
divergence of \S\ref{sec:eval-verify} is precisely a violation of its
analogue by a non-reference engine.
\end{assumption}

\begin{lemma}[Concatenation homomorphism]\label{l:homo}
For unit sequences $U, V$:
$\ME^\ast(U \cat V) = \ME^\ast(U) \cat \ME^\ast(V)$.
\end{lemma}
\begin{proof}
Immediate from the definition of $\ME^\ast$ and
Assumption~\ref{p:fidelity}.
\end{proof}

Table~\ref{t:frontends} records the checked front-end configuration
of the six tokenizers analyzed in this appendix. This set includes
the two WordPiece families and is distinct from the six
GPU-certified configurations of Table~\ref{tab:correctness}. Every
locality statement is made \emph{per configuration}, not for ``HF
tokenizers'' at large.

\begin{table*}[t]
\centering\footnotesize
\begin{tabular}{@{}lllll@{}}
\toprule
Tokenizer & Normalizer & Pre-tokenizer & Added flags &
$\max$ lit. \\
\midrule
Qwen3-8B & NFC & regex (\tok{\textbackslash p\{N\}} single) +
ByteLevel & none & 20 \\
Llama-3.1-8B & none & regex (\tok{\textbackslash p\{N\}\{1,3\}}) +
ByteLevel & none & 30 \\
DeepSeek-V3/V4 & identity & 3$\times$Split + ByteLevel & none & 23 \\
gpt-oss-120b & none & regex (case-aware) + ByteLevel & none & 19 \\
BERT cased & clean\_text, CJK & BertPreTokenizer & none & 6 \\
BERT uncased & + lowercase, accents & BertPreTokenizer & none & 6 \\
\bottomrule
\end{tabular}
\caption{Front ends of the six tokenizers analyzed in this
appendix. ``$\max$ lit.'' =
length in characters of the longest added-token literal, the bound
the leftmost-longest splitter needs. ``none'' = no
\tok{lstrip}/\tok{rstrip}/\tok{single\_word} added tokens. The
unbounded-absorption attack is thereby out of scope \emph{by checked
configuration}, not by argument. DeepSeek-V3 and V4 ship
byte-identical normalizers, pre-tokenizers, and merge tables,
verified by content hash on the frozen snapshots. They differ only
in their added-token lists, so this row (and every base-tokenizer
result in the paper) covers both versions.}
\label{t:frontends}
\end{table*}

\subsection{The splice certificate and the main theorem}

\begin{definition}[Splice certificate]\label{d:cert}
Let $S$ be covered by chunk windows $S_i = S[x_i : y_i)$,
$i = 1..N$, with $x_1 = 0$, $y_N = |S|$, and overlaps
$x_{i+1} < y_i$. A \emph{splice certificate} is a sequence of
junction positions $0 = b_0 < b_1 < \cdots < b_N = |S|$ with
$b_i \in [x_{i+1}, y_i)$ for $0 < i < N$, such that:
\begin{itemize}
  \item[(C1)] no unit of $\FE(S)$ has a source interval straddling
    any $b_i$; write $\FE(S) = A_1 \cat \cdots \cat A_N$ where $A_i$
    collects the units with source interval inside
    $[b_{i-1}, b_i)$;
  \item[(C2)] for each $i$, the units of $\FE(S_i)$ with source
    interval inside $[b_{i-1}, b_i)$ (in global coordinates) are
    \emph{exactly} $A_i$ (same kinds, payloads, and intervals),
    and no unit of $\FE(S_i)$ straddles $b_{i-1}$ or $b_i$.
\end{itemize}
\end{definition}

\begin{theorem}[Losslessness under a splice certificate]\label{t:main}
Under Assumptions~\ref{p:protocol}--\ref{p:fidelity}, if a splice
certificate exists, then concatenating, for $i = 1..N$, chunk $i$'s
records at the token indices produced by its middle block $A_i$
yields exactly $F(S)$.
\end{theorem}

\begin{proof}
By (C1) and Lemma~\ref{l:homo},
\begin{align*}
F(S) &= \ME^\ast(\FE(S)) = \ME^\ast(A_1 \cat \cdots \cat A_N)\\
     &= \ME^\ast(A_1)\cat\cdots\cat\ME^\ast(A_N).
\end{align*}
Fix $i$ and write $\FE(S_i) = L_i \cat A_i \cat R_i$, which is
exactly what (C2) states: $L_i$ (resp.\ $R_i$) are the units of
$\FE(S_i)$ left of $b_{i-1}$ (right of $b_i$), possibly corrupted by
the window edges; the middle block equals $A_i$ verbatim. By
Lemma~\ref{l:homo} chunk $i$ emits
$\ME^\ast(L_i)\cat\ME^\ast(A_i)\cat\ME^\ast(R_i)$, so its records at
indices
\[
[\,|\ME^\ast(L_i)|,\; |\ME^\ast(L_i)|+|\ME^\ast(A_i)|\,)
\]
are $\ME^\ast(A_i)$: $\ME$ depends only on kind and payload
(Assumption~\ref{p:fidelity}), which (C2) pins. Concatenation over
$i$ gives $F(S)$.
\end{proof}

\begin{remark}\label{r:merge-index}
The merge is defined by token \emph{index}. Nothing in
Theorem~\ref{t:main} mentions spans, so duplicate spans, gaps, and
multi-byte grouping are inert. The $N$-chunk case needs no
induction: the unit-block factorization handles all junctions at
once.
\end{remark}

\begin{corollary}[Splice shifting through an equal run]\label{c:shift}
The deployed merge splices at the \emph{end} $c_i$ of a matched run,
not at the certified boundary $b_i$ inside it. Because $[b_i, c_i)$
lies within the equal \tok{(span,id)} run that produced the
certificate, the two chunks' record sequences (write $L$ and $R$
for them) coincide there
($L[b_i{:}c_i] = R[b_i{:}c_i]$), so
\begin{align*}
  L[{:}c_i]\cat R[c_i{:}]
  &= L[{:}b_i]\cat L[b_i{:}c_i]\cat R[c_i{:}]\\
  &= L[{:}b_i]\cat R[b_i{:}c_i]\cat R[c_i{:}]\\
  &= L[{:}b_i]\cat R[b_i{:}],
\end{align*}
splicing at $c_i$ equals splicing at $b_i$. Under the standard
overlap geometry the actual cuts are monotone and non-crossing, so
Theorem~\ref{t:main} certifies at $\{b_i\}$ while the implementation
runs at $\{c_i\}$.
\end{corollary}

\begin{remark}[Cost of the abstraction]
All difficulty lives in discharging (C1)/(C2). Three routes:
(a)~\emph{global front-end pass}: run $\FE$ once, serially, and hand
whole units to workers (the certificate holds by construction, at an
Amdahl cost equal to the front-end share of runtime);
(b)~\emph{local certification} at synchronizing boundaries: check a
bounded window around a candidate $b_i$;
(c)~\emph{matched-run discovery} (the incremental-repair protocol): search
for equal token runs, then obtain a certificate from them, sound for
WordPiece via a continuation-witness certificate, and for byte-BPE
only together with an explicit synchronizing-boundary check
(\S\ref{app:discovery}).
\end{remark}

\subsection{Discharging the certificate per family (summary)}
\label{app:discovery}

The companion document discharges (C1)/(C2) per checked front-end
configuration. For \emph{BERT WordPiece}, boundary discovery from a
matched run is sound with a continuation witness (the deployed
certificate). A matched run containing a non-continuation boundary
pins the pre-tokenizer's state. For the \emph{byte-BPE families}, a
matched ID run alone is \emph{not} a certificate. Context-dependent
digit grouping defeats it (the counterexample of \S\ref{sec:warm}).
The deployed check therefore requires the run to contain a
character-class transition from a per-family \emph{synchronizing
set}, proven to reset the pre-tokenizer under every left context.
The sets are derived per pattern (single-digit vs.\
$\{1,3\}$-digit grouping, case-aware boundaries for o200k). After
the Unicode-version-skew episode of \S\ref{sec:eval-correctness},
every character-class table they consult is probed directly out of
the reference engine rather than any standard library. DeepSeek's
three-splitter sequential composition reduces to the same
per-character predicate. Acceptance on natural text is preserved.
Re-replaying both trace corpora under the certificate configuration
accepts every splice the length-only check accepts
(\S\ref{sec:eval-correctness}).

\section{Workload Characterization: Details}
\label{app:workload}

This appendix carries the full workload characterization
(Fig.~\ref{fig:workload}) and the collection and parsing detail
behind \S\ref{sec:workload}. Nothing here is needed to follow the
paper's argument. All of it is needed to reproduce or audit the
workload numbers.

\begin{figure*}[t]
  \centering
  \includegraphics[width=0.76\textwidth]{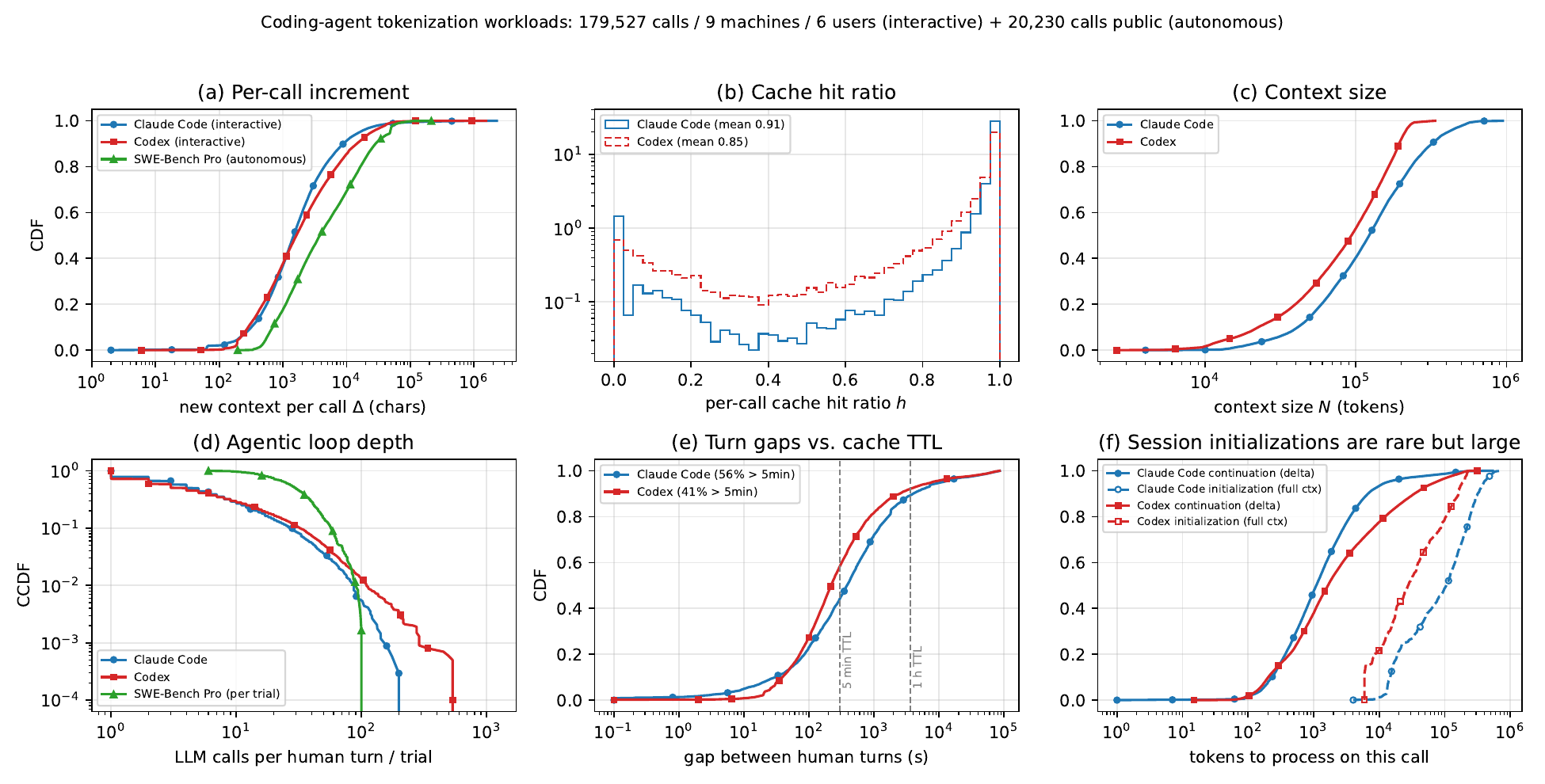}
  \caption{Anatomy of coding-agent tokenization workloads across nine
  machines, two ecosystems, and a public autonomous-agent trace.
  (a)~Per-call increment $\Delta$. (b)~Per-call cache hit ratio $h$.
  (c)~Context size $N$. (d)~LLM calls per
  human turn. (e)~Turn-gap distribution against the 5-minute and
  1-hour cache TTLs. (f)~Tokens-to-process per call, session continuations versus
  session initializations. Continuations process their small delta
  while initializations carry the full context, two orders of magnitude apart. The joint view of
  the (c) and (f) token accounting is
  Fig.~\ref{fig:workload-main} in the main text.}
  \label{fig:workload}
\end{figure*}

\subsection{Collection methodology and privacy discipline (L1)}
\label{app:workload:collection}

The L1 collector is a single-file, zero-dependency parser of the
agents' local session logs. It exports \emph{counts only}. The
fields are token usage as reported by the serving API (input,
cache-read, cache-write where available, output), character counts
of new context (text is length-counted in memory and discarded),
timestamps, model names, and client versions. Session and project
identifiers are HMAC-hashed with a key that never leaves the
contributor's machine. A mechanical self-audit (field whitelist
plus string-shape validation) refuses to package anything else. Contributors authorize collection and inspect
the human-readable output before it leaves their machine. One
contributor's Codex traffic is agent-initiated (their Claude~Code
sessions drive Codex), so its pacing inherits the human loop while
its per-call anatomy sits between our interactive and autonomous
sources.

\subsection{Parsing hazards}
\label{app:workload:hazards}

Three parsing hazards are worth recording for reproducers. First,
both ecosystems' logs report one API call as multiple streamed
records, which must be deduplicated (22{,}274$\to$9{,}987 and
62{,}090$\to$40{,}741 records-to-calls on one machine's data,
respectively). Second, one ecosystem's system-prompt text never
appears in its logs, so character-based increment estimates for first
calls are lower bounds. We therefore report token-based figures
where provenance matters. The third is subtler. One ecosystem's
fork/resume machinery replays the parent session's history
(usage-bookkeeping events included) into the new session's log file,
so a naive cross-file parse double-counts replayed calls. We detect
these replayed prefixes (${\geq}6$ consecutive calls with identical
usage four-tuples, corroborated by degenerate timestamps or a
fork-point context of ${\geq}40$\,K tokens) and drop them. The drop
removes 26{,}578 bookkeeping copies, 14.7\% of parsed calls. A
${\approx}0.8$\% residue of single-record cross-file duplicates in
the other ecosystem's logs survives this rule and is disclosed
rather than removed.

\subsection{The public autonomous-agent trace (L3)}
\label{app:workload:l3}

In \texttt{codex\_swebenchpro}, the release preserves conversation
structure and message lengths, and substitutes a repeating
length-preserving filler passage for the spans it redacts. Every
assistant message is filler. Human and tool messages keep their
prompts, commands, file paths, and captured program output verbatim,
and carry filler in place of the redacted part of their bodies. Most
corpus characters are verbatim text. We disclose this property
wherever it matters to a measurement. Its aggregate cache hit rate (94.2\%) is the figure its
publishers report for this corpus, collected on a workload disjoint
from ours, making it an independent anchor rather than a
re-measurement of our own traces. Its heavier per-call
increments (median 3.8\,K characters) reflect that tool output is
chunkier when no human is pruning it. The SWE-smith
trajectory corpus contributed by the replay campaigns of
\S\ref{sec:eval-correctness} comprises 761\,K requests across
26.1\,K trajectories, of which the certificate-replay campaigns
exercise a 200-stream sample.

\subsection{TraceLab aggregates (L4)}
\label{app:workload:l4}

Our parse of the released TraceLab~\cite{tracelab2026} v0.0.1 trace
reproduces its full
aggregates. They are 4{,}265 sessions from 43 developers,
357{,}161 LLM steps, 432{,}510 tool calls, 54.90\,B input tokens
(52.56\,B cached prefix $+$ 2.34\,B append), and 186.9\,M output
tokens. Its median step carries 126{,}180 cached-prefix tokens
against 857 appended for Claude~Code (115{,}584/886 for Codex). The
token-weighted cache hit rate is 95.7\% overall and 84.4\% on
user-triggered steps. Prefix-cache reads account for 59.5\% of
dollar spend. The per-step retokenization amplification $(P{+}A)/A$
quoted in \S\ref{sec:workload:anatomy}, computed by our script over
that release, splits as median
155$\times$ for Claude~Code and
119$\times$ for Codex, with a pooled P90 of 768$\times$. Because TraceLab
discards text, we use it two ways only. It serves as an external
check on L1--L3, and as a calibration profile for the load
generator driving the closed-loop experiments of
\S\ref{sec:eval-engine}.

\subsection{Temporal detail}
\label{app:workload:temporal}

TraceLab reports provider-side cache behavior consistent with our
client-side turn gaps. Eviction begins at pauses above 5 minutes, and
pauses above an hour almost always miss~\cite{tracelab2026}. Computed
from its released trace, the mean cached share of a step falls from
0.96 (preceding gap under 1 minute) and 0.93 (1--5 minutes) to 0.70
(5--15 minutes), 0.58 (15--60 minutes), and 0.17 beyond an hour, the
per-bucket decay behind the summary in
\S\ref{sec:workload:temporal}.

\vspace{-2.5pt}
\subsection{Note on the GPU efficiency trend}
\label{app:workload:gpucurve}

The 30$\times$ headline for generational inference throughput
(Hopper to Blackwell NVL72) is a vendor-reported comparison against
the same number of H100 GPUs~\cite{nvidia_blackwell}. The mechanisms
that source names are FP4 in a second-generation Transformer Engine
and fifth-generation NVLink across the NVL72 domain. The scissor argument of
\S\ref{sec:workload:scissor} discounts the headline to 5--15$\times$
and needs only its sign.

\section{Sweeping the Append Size}
\label{app:delta}

Table~\ref{tab:warmhead} varies the context while drawing append
sizes from the measured workload (median 1.5--1.6\,K characters per
call), so it shows how each method scales with $N$ but not with
$\Delta$. This appendix sweeps the second axis. It uses fixed
append sizes of 1\,K, 5\,K, 10\,K, 50\,K, and 100\,K characters at
every context shape of Table~\ref{tab:warmhead}, all four methods
on identical texts (Table~\ref{tab:deltasweep}). For calibration,
on the trace pool behind these experiments the per-call increment
has median 1.4\,K characters, P90 6.9\,K, and P99 38\,K. The two
largest settings are past the 99th percentile (0.8\% and 0.5\% of
measured calls reach 50\,K and 100\,K), so they price a stress
regime, not the typical loop.

\begin{table}[t]
  \centering\footnotesize
  \caption{Append-size sweep. P50 latency in ms per append
  ($n{=}24$ per cell), at every context shape of
  Table~\ref{tab:warmhead}. Context sizes (ctx) and append sizes
  ($\Delta$) are in characters. Same host, tokenizer family (Qwen3),
  and timing discipline as Table~\ref{tab:warmhead}. The four
  methods time byte-identical texts. Deltas are fixed-length
  contiguous slices of the same real transcript stream, cut in order
  from each session's append point (controlled length, real text),
  unlike the main table's natural per-call sizes. Gigatoken runs in
  its most favorable mode, cache prewarmed on the session text.
  Every repair cell took the standard repair path, with zero
  retries, no fallback, and 140/140 output spot-checks identical to
  the reference tokenizer.}
  \label{tab:deltasweep}
  \begin{tabular}{@{}lrrrrr@{}}
    \toprule
    ctx $\backslash$ $\Delta$ & 1\,K & 5\,K & 10\,K & 50\,K & 100\,K \\
    \midrule
    \multicolumn{6}{@{}l}{\emph{repair (ours)}} \\
    100\,K & 0.43 & 1.50 & 2.66 & 12.2 & 25.8 \\
    500\,K & 0.76 & 2.28 & 4.01 & 17.2 & 30.2 \\
    1\,M   & 1.29 & 2.73 & 4.41 & 17.0 & 32.2 \\
    2\,M   & 1.67 & 3.18 & 4.71 & 18.2 & 32.7 \\
    3\,M   & 2.14 & 3.71 & 5.80 & 18.6 & 33.6 \\
    4.4\,M & 3.52 & 4.89 & 6.67 & 19.7 & 34.7 \\
    8\,M   & 6.17 & 8.22 & 10.2 & 24.3 & 40.1 \\
    \midrule
    \multicolumn{6}{@{}l}{\emph{Gigatoken, prewarmed cache}} \\
    100\,K & 0.13 & 0.15 & 0.18 & 0.32 & 0.61 \\
    500\,K & 1.18 & 1.19 & 1.21 & 1.38 & 1.54 \\
    1\,M   & 2.52 & 2.54 & 2.46 & 2.75 & 2.86 \\
    2\,M   & 4.78 & 4.79 & 4.84 & 5.05 & 5.25 \\
    3\,M   & 7.34 & 7.37 & 7.42 & 7.51 & 7.64 \\
    4.4\,M & 11.7 & 11.6 & 11.6 & 11.8 & 12.1 \\
    8\,M   & 23.2 & 23.2 & 22.8 & 23.3 & 23.6 \\
    \midrule
    \multicolumn{6}{@{}l}{\emph{fastokens, full retokenization}} \\
    100\,K & 0.73 & 0.81 & 0.90 & 1.68 & 2.73 \\
    500\,K & 4.28 & 4.32 & 4.50 & 5.22 & 6.13 \\
    1\,M   & 9.58 & 8.98 & 9.19 & 10.2 & 10.9 \\
    2\,M   & 21.2 & 17.8 & 17.8 & 18.9 & 20.2 \\
    3\,M   & 31.9 & 28.0 & 26.6 & 27.9 & 28.9 \\
    4.4\,M & 55.4 & 48.9 & 45.1 & 45.8 & 47.9 \\
    8\,M   & 104.7 & 100.2 & 93.5 & 87.5 & 86.6 \\
    \midrule
    \multicolumn{6}{@{}l}{\emph{HF serial, full retokenization}} \\
    100\,K & 22.2 & 23.0 & 24.2 & 36.4 & 52.8 \\
    500\,K & 164.5 & 163.6 & 172.2 & 178.1 & 195.1 \\
    1\,M   & 322.1 & 321.6 & 323.1 & 337.5 & 352.5 \\
    2\,M   & 665.7 & 659.3 & 662.4 & 673.2 & 694.4 \\
    3\,M   & 986.0 & 983.6 & 982.2 & 996.5 & 1008.4 \\
    4.4\,M & 1545.4 & 1537.0 & 1537.0 & 1543.0 & 1561.0 \\
    8\,M   & 3156.5 & 3140.1 & 3150.4 & 3170.9 & 3162.8 \\
    \bottomrule
  \end{tabular}
\end{table}

\paragraph{Protocol.} The context texts are the frozen session
prefixes of Table~\ref{tab:warmhead}. The deltas are fresh
fixed-length slices of the same transcript stream, taken in order
from each session's append point, with the slice positions recorded
in the archive. Each cell has $n{=}24$ samples (4 sessions
$\times$ 6 slices), one timed call per sample per process, on one
pinned core. Each repair sample uses a fresh session. We bootstrap
on the prefix, run one untimed append of the session's first
natural delta (so the timed call measures the steady-state repair
path rather than the one-time first-append bookkeeping visible in
the 8\,M P90 of \S\ref{sec:eval-warm}), and then time the append of
one fixed-size slice. Gigatoken gets a fresh tokenizer object per
cell, prewarmed on the same session text outside the timing region.
The serial engines retokenize the same full text. Under this
protocol the 1\,K column reproduces Table~\ref{tab:warmhead} on its
nearby shapes. The three baselines land within $\pm$7\% of the
archived medians, and repair lands between $-$16\% and $+$7\% (the
controlled 1.0\,K append is smaller than the natural 1.5\,K median,
and the incremental-step protocol excludes the first-append cost that the
main table's mix includes).

\paragraph{Reading.} Repair's cost separates into the two terms its
bound predicts. Down a column sits the $O(N)$ term. At
$\Delta{=}1$\,K the P50 grows from 0.43 to 6.17\,ms as the context
grows 80$\times$, the prefix-verification scan. Across a row sits
the $O(\Delta)$ term. Moving from 1\,K to 100\,K adds 25.4 to
33.9\,ms depending on shape. That is the appended bytes tokenized
at 2.9--3.9\,MB/s, the same regime as the scanned account of
Fig.~\ref{fig:equiv}. The two terms compose near-additively.
Interior cells sit within 17\% of the line through their row's
endpoints (median deviation 7\%). The full retokenizers move only
through the total length $N{+}\Delta$. At the 100\,K shape, a
100\,K append doubles the text and HF's time rises 2.4$\times$. At
1\,M and above, the same append changes the total by at most 10\%,
HF moves by at most 9\%, and fastokens' medians wobble by up to
14\% in both directions, within its usual run-to-run spread. Even
at the degenerate corner where the append equals the context,
repair stays 2.0$\times$ faster than retokenizing the whole text on
the same engine.

The prewarmed Gigatoken column measures a different quantity than
length. Its per-object cache is keyed on content, so an append costs
one rescan plus the handful of pretokens it has not seen before. On
this replay corpus even 100\,K of new transcript is mostly repeated
pretokens, so its P50 at 8\,M moves by under 2\% across the
100$\times$ sweep. The price is content-dependent rather than
length-dependent (novel text pays the full merge path). It comes on
top of the structural properties discussed in \S\ref{sec:eval-warm},
namely process-lifetime keyed memory and the full-context rescan
that makes its floor grow with $N$. The sweep therefore moves the
repair-versus-cache crossover rather than erasing it. At
$\Delta{=}1$\,K (the workload's regime), repair leads from 500\,K
up, as in Table~\ref{tab:warmhead}. At 5--10\,K the crossover moves
to about 2\,M. At 50\,K the prewarmed cache ties repair at 8\,M
(23.3 vs.\ 24.3\,ms) and leads below it. At 100\,K it leads at every
shape, and fastokens too overtakes repair through the 3\,M shape
(10.9\,ms vs.\ 32.2\,ms at 1\,M), with repair regaining the lead at
4.4\,M. The mechanism is plain. Repair wins by scanning less, not
by scanning faster. Once a single append is tens of times the size
the workload actually produces, engines that scan faster per byte,
or charge only for novel content, catch up. Within the measured
append distribution, whose 99th percentile is 38\,K characters, the
ordering of Table~\ref{tab:warmhead} stands.

\section{Displaced Tables and the Served-Equivalent Account}
\label{app:tables}

\label{app:equiv}%
Figure~\ref{fig:equiv} plots the two accounts of
\S\ref{sec:eval-warm} over the public-trace replay. One repair core
climbs from 65\,MB/s at 58\,KB contexts to 1.35\,GB/s at 1.4\,MB
contexts under the served account, while its scanned account holds
at 1.8--3.6\,MB/s.

\begin{figure}[t]
  \centering
  \includegraphics[width=.72\columnwidth]{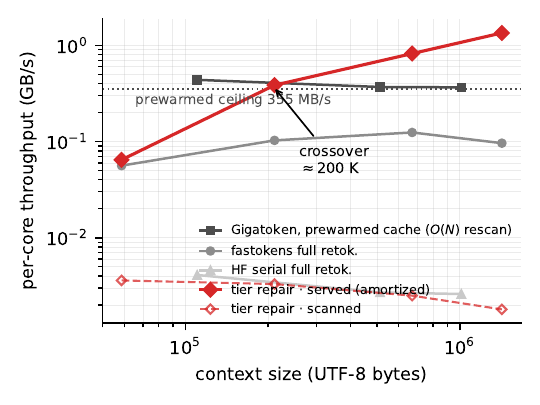}
  \caption{Per-core context throughput on the replay, two accounts
  kept separate (log--log). Under the served account, repair passes
  every $O(N)$ engine as context grows, crossing Gigatoken's
  prewarmed ceiling (355\,MB/s) at ${\sim}200$\,K bytes, while its
  scanned account stays at 1.8--3.6\,MB/s.}
  \label{fig:equiv}
\end{figure}

\label{app:guarantee}%
Table~\ref{tab:guarantee} places the two nearest neighbors on the
guarantee ladder of \S\ref{sec:splicesound}, and
Table~\ref{tab:related} condenses \S\ref{sec:related} onto the two
design axes the system combines.
\label{app:designspace}

\begin{table}[t]
  \centering\footnotesize
  \caption{Guarantee ladder for lossless tokenization claims.}
  \label{tab:guarantee}
  \begin{tabular}{@{}llll@{}}
    \toprule
    & Proof & Per-request check & Runtime guard \\
    \midrule
    Gigatoken~\cite{gigatoken2026} & none & none & none \\
    LoPT~\cite{lopt2026} & in-model$^{\dagger}$ & length threshold & none \\
    \sys & in-model & splice certificate & shadow verifier \\
    \bottomrule
  \end{tabular}
  \\[2pt]{\footnotesize $^{\dagger}$our reproduction found inputs on
  which its published configuration escapes the theorem's premise
  (\S\ref{sec:related}). Its retry/fallback bounds the damage on
  natural text.}
\end{table}

\begin{table}[t]
  \centering\footnotesize
  \caption{Design-space position. ``Exact'' = token-ID-identical to
  the reference implementation, as established by each line of work.}
  \label{tab:related}
  \begin{tabular}{@{}llll@{}}
    \toprule
    & State reuse & Hardware & Exact \\
    \midrule
    Fast CPU tokenizers & none & CPU & is ref.$^{\dagger}$ \\
    GPU tokenizers & none & GPU & no \\
    Incremental BPE & append-only & CPU & under premise \\
    Streaming BPE & bounded delay & CPU & merge stage \\
    LoPT & within-request & CPU & yes \\
    \sys & cross-time, edits & CPU+GPU & certified \\
    \bottomrule
  \end{tabular}
  \\[2pt]{\footnotesize $^{\dagger}$or claims parity with it.
  \S\ref{sec:eval-verify} shows one widely deployed member diverging
  under encode history.}
\end{table}

\begin{figure}[t]
  \centering
  \includegraphics[width=.62\columnwidth]{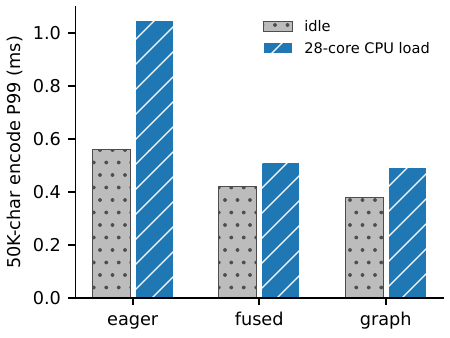}
  \caption{Host-contention companion panel to
  Figure~\ref{fig:tail}. With 28 competing CPU processes on the
  host, CUDA-graph dispatch limits the P99 of a single 50\,K-char
  full tokenization to $+$29\% while the eager pipeline doubles
  (\S\ref{sec:eval-cold}).}
  \label{fig:contention}
\end{figure}

\label{app:econ}%
Table~\ref{tab:econ} carries the cost model behind the rate
normalization of \S\ref{sec:eval-econ}, where, for the CPU front
ends, $\mathbb{E}[t_{tok}]$ in milliseconds reads directly as CPU cores
per 1{,}000 aggregate requests/s and the HuggingFace figure crosses
1{,}100 cores at $r{=}5$ while the tier stays in tens.
Table~\ref{tab:crossgen} carries the five-GPU sweep behind the
clock summary of \S\ref{sec:eval-limits}.
\label{app:crossgen}

\medskip
\noindent\begin{minipage}{\columnwidth}
  \centering\small
  \captionof{table}{Front-end cost under the measured workload mixture, the
  153{,}529 of 153{,}951 collected calls with usable token
  accounting ($\bar N{=}117$K tokens, $\bar h{=}0.88$). ``Per
  1{,}000 GPUs'' is a rate normalization, not a cluster design, and
  topology, model scale, and batching enter only through $r$, the
  request completion rate per inference GPU in requests/s. The first
  four rows are CPU front ends, so their normalized columns are CPU
  cores. $^{\dagger}$The GPU row's 0.15\,ms is GPU service time rather
  than CPU time, so its normalized entries are router-GPU equivalents
  per 1{,}000 inference GPUs, not CPU cores.}
  \label{tab:econ}
  \begin{tabular}{lrrr}
    \toprule
    Front end & $\mathbb{E}[t_{tok}]$ (ms) & cores@$r{=}0.5$ & @$r{=}5$ \\
    \midrule
    HuggingFace fast        & 228.1 & 114.0 & 1{,}140 \\
    Rust (fastokens)        & 13.4  & 6.7   & 67 \\
    Repair over HF engine   & 23.6  & 11.8  & 118 \\
    Repair over Rust engine & 2.3   & 1.1   & 11 \\
    GPU full tokenization$^{\dagger}$ & 0.15  & 0.1   & 1 \\
    \bottomrule
  \end{tabular}
\end{minipage}
\smallskip

\noindent\begin{minipage}{\columnwidth}
  \centering\small
  \setlength{\tabcolsep}{3.5pt}%
  \begin{tabular}{@{}lrrrr@{}}
    \toprule
    & \multicolumn{3}{c}{batch e2e (GB/s)} & 1\,M-char P50 \\
    GPU & eng & cjk & tmpl & (ms, array) \\
    \midrule
    RTX 5090 (Blackwell)      & 4.87 & 4.24 & 4.75 & \textbf{0.383} \\
    RTX PRO 6000 (Blackwell)  & 4.22 & 3.58 & 4.09 & \textbf{0.381} \\
    GH200 (Hopper, Grace host) & 3.91 & 3.11 & 3.79 & 0.464 \\
    H100 NVL (Hopper, x86 host) & 3.96 & 3.11 & 3.83 & 3.90$^{\dagger}$ \\
    A100 (Ampere)             & 2.45 & 1.69 & 2.56 & 5.64$^{\dagger}$ \\
    \bottomrule
  \end{tabular}
  \captionof{table}{Cross-generation portability (Qwen3, every cell backed
  by a byte-identical differential check against the reference
  tokenizer). Throughput follows sustained clock, not SM count or
  bandwidth. The sweep predates the final kernel revisions behind
  the \S\ref{sec:eval-cold} headline numbers, so cross-card
  ordering is the comparison. $^{\dagger}$Python-list delivery,
  which includes the host-interpreter tax of
  \S\ref{sec:eval-cold}.}
  \label{tab:crossgen}
\end{minipage}

\section{Serving Large Appends: Three Explored Routes}
\label{app:routes}

Appendix~\ref{app:delta} ends with a lost regime. At
$\Delta{=}100$\,K characters the fully prewarmed cache leads repair
at every context shape. Appends this large are rare, since the
measured P99 is 38\,K characters and 0.8\% of calls reach 50\,K,
but a serving tier should still have an answer for them. This
appendix details three routes we designed and measured to reclaim
that regime. All three are exploration prototypes rather than
shipped defaults, and none has passed the certification battery
that gates the shipped repair path
(\S\ref{sec:eval-correctness}). Each prototype does keep the
exactness discipline. Every timed sample asserts its output IDs
against full reference tokenization, and any divergence aborts the
run. Unless stated otherwise, numbers are P50 milliseconds under
the host, protocol, texts, and Qwen3 family of
Appendix~\ref{app:delta} ($n{=}24$ per cell), and window-max is
the 4.4\,M-character shape. Figure~\ref{fig:routing-phase} in this
appendix plots the resulting decision surface, referenced from the
routing discussion of \S\ref{sec:conclusion}. This
appendix carries the designs, measurements, and verdicts behind
it.

\subsection{Route A: parallel repair over the append (retired)}
\label{app:routes:a}

At $\Delta{=}100$\,K the shipped repair path takes 32--40\,ms
(Table~\ref{tab:routes}), almost entirely one serial encode of the
appended text. LoPT~\cite{lopt2026} splits one request into chunks
for CPU parallelism, and our splice certificate is already proved
for $N$ chunk windows (Definition~\ref{d:cert},
Theorem~\ref{t:main}). Route A combines the two. It splits the
append into $k$ chunks, tokenizes them concurrently, and admits
every seam with the same defensive match and certified-boundary
check as deployed incremental repair. A seam that fails the check
widens its overlap, and repeated failure collapses the whole
append back to serial tokenization.

Two implementation facts bound the speedup. First, the reference
engine holds the interpreter lock during encode. Sixteen threads
tokenize eight 100\,K chunks no faster than one thread (25.7
vs.\ 25.6\,ms), so workers must be processes, and fork-preloaded
workers with array transfer cost 0.3\,ms of dispatch at $k{=}2$
and 2.0\,ms at $k{=}16$. Second, seam matching stays serial Python
at about 0.3\,ms per seam. At the 1\,M shape with
$\Delta{=}100$\,K, $k{=}8$ overlaps encode down to 7.0\,ms and
pays 2.1\,ms of match. Moving to $k{=}16$ buys back 1.0\,ms of
encode while match doubles to 4.4\,ms, and wall time regresses
beyond $k{=}8$ at every shape. The certificate checks themselves
are nearly free at 0.14--0.24\,ms per repair.

The prototype works and is exact. All 1{,}104 timed samples match
the reference IDs and byte spans. Across the full sweep, all
6{,}384 seams accept on the first attempt, with zero retries and
zero collapses to serial. Real transcripts therefore carry
certified boundaries densely enough for chunk-parallel repair.
Eight cores cut $\Delta{=}100$\,K repair to 9.8, 11.3, and
13.1\,ms at 1\,M, window-max, and 8\,M, a 3.1--3.3$\times$ gain
over the shipped path, and simple variants (a 256-character
initial overlap, deferred text materialization) reach 11.4\,ms at
8\,M. Of the six cells measured here ($\Delta$ of 50\,K and
100\,K at the three shapes), that wins back all four window-max
and 8\,M cells, and both 1\,M cells stay lost (9.8 vs.\
2.9\,ms). We
retired the route anyway. Route B reaches 2.4--3.7$\times$ lower
latency on the Table~\ref{tab:routes} cells using one core instead
of eight. What survives is evidence about the certificate itself.
The deployed system applies one certificate across time. Route A
applies many across space within a single request, at a 100\%
first-try acceptance rate on real text.

\subsection{Route B: a cache-based tokenizer as window engine}
\label{app:routes:b}

Route A parallelizes the shipped window encoder. Route B replaces
it. Gigatoken's cache-based engine tokenizes repeated text far
faster than the reference engine, but it returns token IDs only,
and repair state needs byte spans (\S\ref{sec:state}). For
byte-level BPE families the byte length of every token is fixed by
the vocabulary, so spans are reconstructable from IDs alone. A
prefix sum over per-token byte lengths gives byte offsets, and a
byte-to-character map lifts them to character spans. The
reconstruction runs behind four guards. Family admission accepts
byte-level BPE only. The byte-length table is derived from two
independent sources and compared entry by entry. Every call checks
that the reconstructed offsets exactly cover the window. For the
NFC-normalized family (Qwen3), a segmented invariance guard
verifies that normalization does not rewrite the window, at
0.43\,ms on a 100\,K window, and refuses windows it would rewrite.
A refusal falls back to the reference engine. Exactness therefore
comes from the tier's own checks, not from the engine.

Reconstructed spans match the reference tokenizer's offsets token
by token on four byte-BPE families, 1.19--1.30\,M tokens per
family, over multilingual, code, adversarial, and real-transcript
suites. We then replayed 17{,}066 certified splices and 1{,}800
adversarial edits through the full incremental repair path with the
engine
swapped in, in lock step with the reference engine, and observed
zero divergence. The only nonzero counter is 44 guard refusals,
all from one stream about Unicode transliteration whose decomposed
sequences NFC rewrites, and each refusal fell back with IDs still
equal. We word this as a pilot, not an admission. Admission would
require the full battery of \S\ref{sec:eval-correctness} (92{,}484
splices and 15{,}000 edits) plus engine-level integration of the
fallback semantics.

The swap removes the $O(\Delta)$ bottleneck. Appended bytes
tokenize at 107--128\,MB/s at the 1\,M and window-max shapes and
94\,MB/s at 8\,M, against 2.9--3.9\,MB/s on the shipped path,
roughly a 30--40$\times$ higher delta term. At $\Delta{=}100$\,K
Route B takes 2.66, 3.79, and 5.34\,ms at 1\,M, window-max, and
8\,M on one core. It matches the fully prewarmed cache at 1\,M
(2.86\,ms) and beats it by 3.2$\times$ at window-max and
4.4$\times$ at 8\,M, because it inherits repair's structure and
rescans nothing outside the window. Small appends do not regress,
and the 1\,K rows (1.5--2.2\,ms) sit in the range of the shipped
store's archived 1\,K rows (1.3--6.2\,ms). Two costs frame
deployment. A fresh engine object takes about 180\,ms to
construct, outside the timing region, and a long-lived object
reintroduces the content-keyed memory growth discussed in
\S\ref{sec:eval-warm}. Combining Routes A and B is mechanically
sound, with every combined-run assertion passing, but pointless at
these sizes. Window encodes are already sub-millisecond, so pool
dispatch dominates and the $k{=}8$ combination stays flat at
3.1--3.2\,ms (1\,M) and 5.4--5.8\,ms (8\,M) across all append
sizes.

\subsection{Route C: GPU full tokenization with span export}
\label{app:routes:c}

A large append can also be served by full tokenization. The GPU path
re-encodes base plus append in 1.08--1.13\,ms at 1\,M and
2.82--3.18\,ms at window-max over both measured appends (50\,K and
100\,K), the cells behind the 6--28$\times$ rerouting estimate of
\S\ref{sec:conclusion}, and 1.1--2.4$\times$ faster than the best
CPU route in the same cell. At 8\,M the picture splits by
normalization. Llama~3.1 stays flat at 6.1--6.2\,ms. Qwen3
requests fork on the NFC quick check (\S\ref{sec:eval-cold}) into
three service classes of about 6, 36, and 260\,ms, for a pass, a
CPU adjudication that confirms the text unchanged, and an actual
rewrite. Sessions that hit the slow branches belong on a CPU
route. The hybrid split is measured, not hypothetical.

Routing to full tokenization had a blocker. The kernels emit IDs without byte
spans, so seeding session state required a reference rebuild
costing 407, 2{,}067, and 4{,}112\,ms at the three shapes. The
rebuild row includes offset materialization and state
construction, so it exceeds a bare full encode. The measured
traces would even let a background rebuild hide. After a 50\,K+
append, the probability that the session's next call arrives
inside the rebuild window is 0.3--4.3\% across shapes, for an
expected extra cost below 1.6\,ms per event, at the price of one
CPU core busy for 0.4--4.1\,s. Span export removes the account
entirely. Route B's byte-length reconstruction runs on the device.
IDs gather per-token byte lengths, a cumulative sum yields byte
offsets, a second cumulative sum over non-continuation text bytes
lifts them to character spans, and one packed transfer returns IDs
and spans together. The shipped kernels are unchanged, the
prototype reuses Route B's guards, and exporting spans adds only
0.13--0.73\,ms to the IDs-only channel up to window-max. For Qwen3
the existing NFC quick check is exactly the guard the
reconstruction needs. A pass or an identity adjudication implies
the spans are original-text coordinates, and a rewrite implies
refusal and fallback. The measured refusal surface matches Route
B's CPU guard case for case.

Exported spans match the reference on 2{,}544{,}721 tokens across
the two families, with zero ID and zero offset mismatches, and on
288 full-scale route samples of 1.06--8.1\,M characters. End to
end, a GPU full tokenization with span export seeds session state that
then serves five real appends per session. All 345 replayed append
steps and all 69 end-of-stream oracle checks equal the
bootstrapped reference, and the 3 refused seedings are Qwen3
rewrite sessions at 8\,M. Seeding costs 1.8--8.3\,ms on the GPU
plus 0.011\,ms to wrap the arrays for the repair engine, about
230--500$\times$ below the reference rebuild. Building the boxed
Python session state from the same arrays instead costs a further
0.2--1.0\,s, a cost the Rust store form does not pay.
\S\ref{sec:discussion} lists GPU span export as the
highest-leverage missing piece of the shipped system. The
prototype shows the mechanism is sound and cheap. What remains is
battery-scale certification, added-token handling, and
adapter-level fallback semantics. The leverage also exceeds this
appendix's tail case, because every full tokenization that seeds
session state currently pays the reference rebuild.

\subsection{Verdict}
\label{app:routes:verdict}

\begin{figure}[t]
  \centering
  \includegraphics[width=.94\columnwidth]{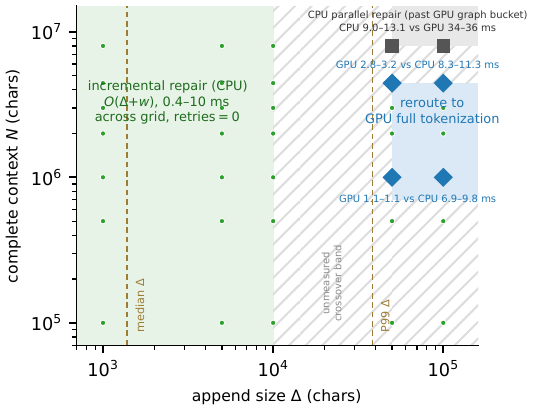}
  \caption{Measured routing phase diagram on the append-size
  $\times$ context plane, under exploration-prototype accounting
  rather than shipped defaults. Small green points are incremental repairs
  from the append-size sweep, and the six large markers are cells
  where a GPU full-tokenization rebuild and the best CPU repair option
  were
  measured head to head. Real appends concentrate far left of the
  GPU region, and past the GPU graph-capture bucket a parallel CPU
  repair prototype wins again. Hatching marks unmeasured bands.}
  \label{fig:routing-phase}
\end{figure}

\begin{table}[t]
  \centering\footnotesize
  \caption{The three routes head to head at $\Delta{=}100$\,K
  characters (P50 ms, Qwen3, $n{=}24$ per cell, same texts and
  protocol as Table~\ref{tab:deltasweep}, prototype accounting).
  ``Shipped'' is the deployed repair path and ``cache'' is
  Gigatoken fully prewarmed, both archived from the
  Appendix~\ref{app:delta} sweep. Bold marks the fastest route per
  shape. $^{\dagger}$NFC quick-check fast branch. The mixed P50
  over sessions that hit the slow branches is 36.3\,ms (see
  text).}
  \label{tab:routes}
  \setlength{\tabcolsep}{4.0pt}%
  \begin{tabular}{@{}lrrrrr@{}}
    \toprule
    ctx & shipped & A ($k{=}8$) & B (1 core) & C (GPU) & cache \\
    \midrule
    1\,M    & 32.2 & 9.80 & 2.66 & \textbf{1.13} & 2.86 \\
    4.4\,M  & 34.7 & 11.3 & 3.79 & \textbf{2.82} & 12.1 \\
    8\,M    & 40.1 & 13.1 & \textbf{5.34} & 6.05$^{\dagger}$ & 23.6 \\
    \bottomrule
  \end{tabular}
\end{table}

Table~\ref{tab:routes} settles the regime that
Appendix~\ref{app:delta} lost. Up to window-max the GPU route is
fastest. At 8\,M a CPU repair prototype wins, the parallel route
at $\Delta{=}50$\,K (9.0\,ms) and the window engine at
$\Delta{=}100$\,K, with the GPU fast branch within 1.2$\times$.
Figure~\ref{fig:routing-phase} draws the same surface with the
parallel prototype as the CPU option, and this table adds Route B,
which improves the 8\,M cell further. Every cell the prewarmed
cache won in Table~\ref{tab:deltasweep} is reclaimed by at least
one route, and the routes compose into a natural hybrid. Large
appends go to the GPU, NFC slow sessions and the largest shapes go
to a CPU repair route, and span export seeds session state either
way.

Perspective still matters. Appends of 50\,K characters and above
are 0.8\% of measured calls, so these routes price a tail, not the
loop that \S\ref{sec:eval} measures, and the shipped default
remains repair for every session continuation until a rerouted path passes
the same certification battery. What this appendix establishes is
that the crossover of Appendix~\ref{app:delta} is not structural.
It reflects the shipped window encoder, and three independent,
measured mechanisms move it.

\end{document}